\documentclass[11pt]{article}
\RequirePackage[OT1]{fontenc}
\RequirePackage{amsthm,amsmath,amssymb,amsfonts,subcaption,graphicx,epstopdf,enumerate,lmodern}
\RequirePackage[round,colon,authoryear]{natbib}
\RequirePackage{bigints}
\RequirePackage[colorlinks,citecolor=blue,urlcolor=blue]{hyperref}
\usepackage{bm,color,fancyvrb,xcolor,mathtools,mathbbol}
\usepackage{amscd}
\usepackage{mathrsfs}
\usepackage{bbm}
\usepackage{booktabs} 
\usepackage{tikz}
\usetikzlibrary{arrows.meta}
\usepackage{algorithm}
\usepackage{algpseudocode}

\newtheorem{theorem}{Theorem}
\newtheorem{corollary}{Corollary} 
\newtheorem{lemma}{Lemma}

\newtheorem{remark}{Remark}
\newtheorem{proposition}{Proposition}
\newtheorem{definition}{Definition}
\newtheorem{example}{Example}
\newtheorem{convention}{Convention}
\makeatletter
\newcounter{subassumption}[assumption]  

\usepackage{amsmath}
\usepackage{xcolor}
\usepackage{pgfplots}
\pgfplotsset{compat=1.16}

\usepackage{booktabs}
\usepackage{rotating}
\usepackage{tabularx}
\usepackage{array}

\newcommand{\R}{\mathbb R}
\newcommand{\Rnn}{\mathbb R_{\ge 0}}
\newcommand{\rank}{\operatorname{rank}}
\newcommand{\rankp}{\operatorname{rank}_{+}}
\newcommand{\supp}{\operatorname{supp}}
\newcommand{\relint}{\operatorname{relint}}
\newcommand{\spann}{\operatorname{span}}
\newcommand{\cone}{\operatorname{cone}}
\newcommand{\conv}{\operatorname{conv}}
\newcommand{\T}{\mathcal T}

\makeatother

\begin{document}

\title{Identifiability of Nonnegative Tensor Decompositions via Positive Scattering}

\author{Haoming Wang and Ming Yuan\\ Columbia University}

\maketitle

\begin{abstract}
Identifiability of tensor decompositions is often established through
linear-algebraic conditions on the factor families. For nonnegative
decompositions, however, positivity provides additional information that is
not captured by dimension and independence alone: nonnegative terms cannot
cancel, and their supports constrain competing decompositions. We introduce a
positive scattering term that quantifies this additional source of
identifiability and combine it with the dimension budget underlying the
Lovitz--Petrov generalization of Kruskal's theorem. For every subset of
components, we obtain two sufficient conditions: a threshold of $2|S|-2$
guarantees minimality and nonnegative rank, while the stronger threshold
$2|S|-1$ guarantees uniqueness among nonnegative decompositions of the same
length. The key result is a positive splitting inequality for irreducible
exchanges of nonnegative rank-one tensors, which combines the dimension
constraint with support-induced geometric rigidity. Although the scattering
term is defined through an optimization over intermediate factor spaces, we
show that its mode costs are exactly $0$, $1$, or $+\infty$, yielding an exact
activation characterization in terms of graph connectivity. The resulting
criterion can strictly certify sparse nonnegative tensor decompositions beyond
the reach of Kruskal and Lovitz--Petrov conditions, including examples for
which those conditions fail even after reshaping. In the matrix case, the two
criteria reduce respectively to full-rank factorization and two-sided
separability.

\end{abstract}

\section{Introduction}\label{sec:intro}

Identifiability asks whether observable multilinear data uniquely determine
their latent rank-one components, up to the unavoidable permutation and scaling
ambiguities.  This question is central whenever a tensor decomposition is used
as a structural model rather than merely as a numerical approximation.  Tensor
methods, for example, turn low-order observable moments into latent-component
recovery procedures in mixture, topic, and other latent-variable models
\citep{allman2009identifiability,AnandkumarEtAl2014,AnandkumarEtAl2015}.
Nonnegative tensor decompositions arise in such models when the latent
components represent quantities that are intrinsically nonnegative.

Nonnegative tensor decompositions also arise naturally in signal processing.
They have been used, for example, for blind audio source separation
\citep{BarkerVirtanen2016} and multilinear spectral unmixing of hyperspectral
data \citep{VeganzonesEtAl2016}.  More broadly, tensor decompositions provide
identifiability in blind source separation and related multilinear inverse
problems.  Deterministic uniqueness conditions for canonical polyadic
decompositions have therefore been developed to exploit additional structure
in signal-processing models, including known, orthonormal, or partially
Hermitian factors \citep{SorensenDeLathauwer2015}.  Nonnegativity provides a
different form of structural information: unlike orthogonality or symmetry,
it constrains competing decompositions through the absence of cancellation.
Zeros and supports in the observed tensor therefore carry information about
which alternative components are possible.

A substantial literature studies uniqueness and identifiability of tensor
decompositions through linear-algebraic properties of their factor families.
Kruskal's theorem uses the Kruskal ranks of the factors
\citep{Kruskal1977}, while the Lovitz--Petrov theorem replaces these global
conditions by a subset-wise dimension budget and strictly generalizes the
Kruskal condition \citep{LovitzPetrov2023}.  For nonnegative tensors, related
work has established existence of best nonnegative low-rank approximations
\citep{LimComon2009} and generic uniqueness of such approximations
\citep{QiComonLim2016}.  These results leave open a different deterministic
question: given a specified exact nonnegative decomposition, how much
additional identifiability can be obtained from nonnegativity itself, beyond
what is captured by the dimensions of the factor spans?

This paper develops a deterministic answer to this question.  Our main idea
is to separate two sources of identifiability.  The first is linear-algebraic
and is measured by the Lovitz--Petrov dimension budget
\[
\beta(S)=\sum_{j=1}^d\bigl(d_j(S)-1\bigr),
\]
for subsets $S$ of components.  The second comes from positivity and support
geometry.  We quantify it by a \emph{positive scattering term} $\tau(S)$,
whose formal definition is given in Section~\ref{sec:tau}.  Informally,
$\tau(S)$ measures the minimum additional factor-space dimension required to
connect the prescribed components through the support geometry imposed by
nonnegativity.  Thus $\beta(S)$ and $\tau(S)$ capture two distinct sources of
identifiability that can be combined in a single certificate.

Our main result gives two deterministic sufficient conditions.  If
\[
\beta(S)+\tau(S)\ge 2|S|-2
\]
for every $S\subseteq[R]$ with $|S|\ge2$, then the prescribed decomposition is
minimal and its number of terms equals the nonnegative rank.  If the threshold
is strengthened by one, to
\[
\beta(S)+\tau(S)\ge 2|S|-1,
\]
then the decomposition is unique among nonnegative decompositions of the same
length.  The key structural result is a \emph{positive splitting inequality}:
for every irreducible exchange between $p$ prescribed nonnegative rank-one
terms and $q$ competing terms,
\[
\beta(S)+\tau(S)\le p+q-2.
\]
The extra term $\tau(S)$ arises because an irreducible positive exchange must
remain connected under the support geometry of the factor spaces, while the
Lovitz--Petrov argument controls the corresponding linear dimension.  The two
identifiability thresholds then follow by contradiction from the fact that a
shorter competing decomposition creates an unbalanced exchange, whereas an
inequivalent minimal decomposition creates a balanced exchange of size at
least two.

Although $\tau(S)$ is introduced through a seemingly continuous optimization
over intermediate factor spaces, it admits an exact finite characterization.
We show that the cost associated with each mode is always $0$, $1$, or
$+\infty$, and that $\tau(S)$ is the minimum number of modes that must be
activated to make an associated graph connected.  This yields an exact
computational procedure based on support tests, linear-programming vertex
tests, and graph connectivity.  The resulting criterion is strictly stronger
than the corresponding dimension-based criteria: we give deterministic
families for which the new condition certifies nonnegative uniqueness even
though the Lovitz--Petrov condition fails, including after reshaping, and
numerical experiments show substantial gains in sparse regimes.

The theory also clarifies the matrix boundary case.  For order two, the
minimality criterion reduces exactly to full column rank of the two factors,
or equivalently to $\operatorname{rank}(X)=R$, while the uniqueness criterion
reduces exactly to two-sided separability.  Thus the same positive-scattering
mechanism that strengthens tensor identifiability specializes to a classical
nonnegative matrix identifiability condition.

The remainder of the paper is organized as follows.  Section~\ref{sec:main}
states the problem and the main criterion at a high level.  Section~\ref{sec:exchanges}
develops positive exchanges and the Lovitz--Petrov splitting mechanism.
Section~\ref{sec:cones} develops the support geometry and bridge connectivity
induced by nonnegativity.  Section~\ref{sec:tau} gives the formal definition
of the positive scattering term and derives its exact activation
characterization.  Section~\ref{sec:proofs} proves the positive splitting
inequality and the main identifiability result.  Section~\ref{sec:structure} develops reshaping
and appending-mode consequences, Section~\ref{sec:examples} gives deterministic
and numerical examples, and Section~\ref{sec:matrix-case} specializes the
theory to matrices.

\section{Problem Setup and Main Criterion}\label{sec:main}

Write $[n]=\{1,\dots,n\}$. Fix $d\ge2$ and dimensions
$n_1,\dots,n_d\ge1$, and consider a nonnegative decomposition
\begin{equation}
\T=\sum_{r=1}^{R}P_r,
\qquad
P_r=a_r^{(1)}\otimes\cdots\otimes a_r^{(d)},
\qquad
a_r^{(j)}\in\Rnn^{n_j}\setminus\{0\}.
\label{eq:decomposition}
\end{equation}
The \emph{nonnegative rank} $\rankp(\T)$ is the smallest number of
nonzero nonnegative rank-one tensors whose sum is $\T$; a nonnegative
decomposition is \emph{minimal} if its length equals $\rankp(\T)$.
For $d=2$, this additive rank-one representation is equivalent to the
usual nonnegative matrix factorization formulation
\citep[p.~152]{CohenRothblum1993}. Two nonnegative decompositions of the
same length are \emph{equivalent} if their multisets of rank-one terms
coincide.

The support of a vector $x\in\R^n$ is
$\supp(x)=\{i\in[n]:x_i\neq0\}$; the support of a tensor is defined
coordinatewise in the same way. For a nonzero nonnegative tensor, the sum
of all entries, called the \emph{entry sum}, is strictly positive. We will
use repeatedly the elementary fact that a sum of nonzero nonnegative
tensors is never zero.

For a nonempty subset $S\subseteq[R]$ and each mode $j$, put
\[
U_j(S)=\spann\{a_r^{(j)}:r\in S\},
\qquad
d_j(S)=\dim U_j(S)\ge1,
\]
and define the \emph{Lovitz--Petrov dimension budget}
\citep[Theorem~2]{LovitzPetrov2023}
\[
\beta(S)=\sum_{j=1}^{d}\bigl(d_j(S)-1\bigr)\ge0.
\]

The main result augments this dimension budget with a positive scattering
term $\tau(S)$. We defer its formal definition to Section~\ref{sec:tau}.
Informally, $\tau(S)$ measures the additional factor-space enlargement
forced by nonnegativity to connect the prescribed components through their
support geometry. Thus $\beta(S)$ captures linear-algebraic information,
while $\tau(S)$ captures additional rigidity arising from positivity and
support geometry.

We use the following two conditions:
\begin{equation}
\beta(S)+\tau(S)\ge 2|S|-2,
\qquad
\text{for every }S\subseteq[R]\text{ with }|S|\ge2,
\tag{M}\label{eq:M}
\end{equation}
and
\begin{equation}
\beta(S)+\tau(S)\ge 2|S|-1,
\qquad
\text{for every }S\subseteq[R]\text{ with }|S|\ge2.
\tag{U}\label{eq:U}
\end{equation}
The first threshold rules out shorter competing nonnegative
decompositions, while the stronger threshold also rules out competing
decompositions of the same length.

\begin{theorem}[Positive Lovitz--Petrov criterion]
\label{thm:main}
Let \eqref{eq:decomposition} be a nonnegative decomposition.
If condition \emph{(M)} holds, then
\[
\rankp(\T)=R,
\]
so that \eqref{eq:decomposition} is minimal and its terms are linearly
independent. If the stronger condition \emph{(U)} holds, then every
nonnegative decomposition of $\T$ of length $R$ is equivalent to
\eqref{eq:decomposition}.
\end{theorem}

The structural result behind Theorem~\ref{thm:main} is a positive splitting
inequality. For an irreducible positive exchange involving $p$ prescribed
terms and $q$ competing nonnegative rank-one terms, it gives
\[
\beta(S)+\tau(S)\le p+q-2.
\]
The formal statement and proof are given in
Section~\ref{sec:proofs}. The key point is that the usual
Lovitz--Petrov dimension argument controls the linear complexity of the
exchange, while nonnegativity imposes additional support constraints that
must be paid for through $\tau(S)$.

The criterion automatically contains the Lovitz--Petrov criterion because
$\tau(S)\ge0$. The improvement can be strict: Section~\ref{sec:examples}
gives explicit nonnegative decompositions satisfying \emph{(U)} for which
the Lovitz--Petrov condition fails, even after reshaping the tensor.

For comparison, the three criteria considered in this paper differ in the
structural information they use:
\begin{center}
\begin{tabular}{lll}
\toprule
Criterion & Information used & Guarantee \\
\midrule
Kruskal & Kruskal ranks & CP uniqueness \\
Lovitz--Petrov & subset-wise factor dimensions & CP uniqueness \\
This paper & dimensions + support geometry & nonnegative minimality/uniqueness \\
\bottomrule
\end{tabular}
\end{center}

\section{Positive Exchanges and the Splitting Mechanism}\label{sec:exchanges}

This section develops the algebraic mechanism underlying the main
identifiability result.  We first record two elementary facts about
nonnegative rank-one tensors.  We then introduce positive exchanges and
decompose them into irreducible blocks.  Finally, we encode an exchange as a
signed family of product tensors and recall the Lovitz--Petrov splitting
theorem.  The latter provides the linear-algebraic part of the argument; the
additional constraint induced by nonnegativity will be developed through
support geometry in the next sections.

\subsection{Basic Facts for Nonnegative Decompositions}

We begin with two elementary facts about product tensors.  The first gives
the usual uniqueness of factorization up to rescaling, while the second shows
that a nonnegative rank-one tensor always admits nonnegative factors.

\begin{lemma}[Product tensors]\label{lem:product-tensors}
\leavevmode
\begin{enumerate}[(i)]
\item\label{it:rigidity}
If $x_j,y_j\in\R^{n_j}$ satisfy
$x_1\otimes\cdots\otimes x_d=y_1\otimes\cdots\otimes y_d\neq0$, then there
exist nonzero scalars $\lambda_1,\dots,\lambda_d$ with
$y_j=\lambda_j x_j$ for every $j$ and $\prod_{j=1}^d\lambda_j=1$.
If in addition all the vectors $x_j,y_j$ are nonnegative, then every
$\lambda_j$ is positive.
\item\label{it:nonneg-factors}
Every nonzero nonnegative rank-one tensor is a tensor product of nonzero
nonnegative vectors.
\end{enumerate}
\end{lemma}

\begin{proof}
\emph{Part (i).}
Write $X=x_1\otimes\cdots\otimes x_d$.  Since $X\neq0$, there is a
multi-index $(i_1^*,\dots,i_d^*)$ with
$X_{i_1^*\cdots i_d^*}\neq0$.  The entries of $X$ are the products
$\prod_{j=1}^d x_j(i_j)$, so $x_j(i_j^*)\neq0$ for every $j$, and likewise
$y_j(i_j^*)\neq0$ for every $j$.

Fix a mode $j$ and let $i\in[n_j]$ be arbitrary.  Comparing the entry of
the two product tensors whose $j$th index is $i$ and whose $l$th index is
$i_l^*$ for every $l\neq j$ gives
\[
x_j(i)\prod_{l\neq j}x_l(i_l^*)
=
y_j(i)\prod_{l\neq j}y_l(i_l^*).
\]
The two products over $l\neq j$ are nonzero constants independent of $i$.
Hence $y_j=\lambda_jx_j$ with
\[
\lambda_j=
\frac{\prod_{l\neq j}x_l(i_l^*)}
{\prod_{l\neq j}y_l(i_l^*)}\neq0.
\]
Substituting back,
\[
y_1\otimes\cdots\otimes y_d
=
\Bigl(\prod_{j=1}^d\lambda_j\Bigr)
x_1\otimes\cdots\otimes x_d,
\]
and since this equals $X\neq0$ we obtain
$\prod_{j=1}^d\lambda_j=1$.

Now suppose all vectors are nonnegative.  Choose $i$ with $y_j(i)>0$.
Then $\lambda_jx_j(i)=y_j(i)>0$ forces $x_j(i)\neq0$, hence
$x_j(i)>0$ and $\lambda_j>0$.

\emph{Part (ii).}
Let
\[
X=b_1\otimes\cdots\otimes b_d\neq0
\]
be entrywise nonnegative, with real factors $b_j$.  Choose a multi-index
$(i_1^*,\dots,i_d^*)$ with
\[
X_{i_1^*\cdots i_d^*}
=
\prod_{j=1}^d b_j(i_j^*)>0.
\]
In particular $b_j(i_j^*)\neq0$ for all $j$, and the number of indices
$j$ for which $b_j(i_j^*)<0$ is even.  Multiplying an even number of
factors by $-1$ leaves the product tensor unchanged, so after flipping
signs in pairs we may assume
\[
b_j(i_j^*)>0
\qquad\text{for every }j.
\]

We claim that then each $b_j$ is nonnegative.  Fix $j$ and
$i\in[n_j]$, and consider the entry of $X$ whose $j$th index is $i$ and
whose other indices are $i_l^*$:
\[
0\le
X_{i_1^*\cdots i\cdots i_d^*}
=
b_j(i)\prod_{l\neq j}b_l(i_l^*),
\]
where the product over $l\neq j$ is strictly positive.  Hence
$b_j(i)\ge0$.  Each $b_j$ is also nonzero, since $X\neq0$.
\end{proof}

\begin{convention}\label{conv:nonneg-factors}
By Lemma~\ref{lem:product-tensors}\,\eqref{it:nonneg-factors}, we may and
will choose nonnegative factors for every nonzero nonnegative rank-one tensor
under consideration.  By Lemma~\ref{lem:product-tensors}\,\eqref{it:rigidity},
this choice is unique up to positive rescalings whose product is one.
All subsequent quantities are invariant under these rescalings.
\end{convention}

The next lemma records a basic consequence of nonnegativity that will also
be used in the proof of the main criterion.

\begin{lemma}\label{lem:minimal-independent}
The terms of a minimal nonnegative decomposition are linearly independent.
\end{lemma}

\begin{proof}
Let
\[
\T=\sum_{r=1}^{R}P_r
\]
be minimal and suppose that
\[
\sum_{r=1}^{R}c_rP_r=0
\]
is a nontrivial linear relation.

The nonzero coefficients cannot all have the same sign.  Indeed, after
negating the relation if necessary we may assume that all nonzero $c_r$
are positive.  Taking the entry sum of both sides then gives
\[
0=\sum_r c_r\,(\text{entry sum of }P_r)>0,
\]
a contradiction, since the entry sum of each $P_r$ is positive.  Hence,
after negating the relation if necessary, at least one coefficient is
positive and at least one is negative.

Put
\[
\varepsilon
=
\min_{r:\,c_r>0}\frac{1}{c_r}>0,
\]
attained at some index $r_0$.  For every $r$ the coefficient
$1-\varepsilon c_r$ is nonnegative: this is clear when $c_r\le0$, and
when $c_r>0$ it follows from $\varepsilon c_r\le1$.  Moreover,
$1-\varepsilon c_{r_0}=0$.  Consequently,
\[
\T
=
\sum_{r=1}^{R}P_r
-\varepsilon\sum_{r=1}^{R}c_rP_r
=
\sum_{r=1}^{R}(1-\varepsilon c_r)P_r
\]
is a sum of at most $R-1$ nonzero nonnegative rank-one tensors, after
absorbing each positive coefficient into one factor and discarding terms
with coefficient zero.  This contradicts the minimality of $R$.

\end{proof}

\subsection{Positive Exchanges and Irreducible Blocks}

We now introduce the basic object used to compare two nonnegative
decompositions.

\begin{definition}[Positive exchange]\label{def:exchange}
A \emph{positive exchange} is an equality
\[
\sum_{i\in I}X_i
=
\sum_{s\in J}Y_s
\]
between two finite sums of nonzero nonnegative rank-one tensors, indexed by
finite sets $I$ and $J$.  A \emph{subexchange} of it is a pair
$(I',J')$ with $I'\subseteq I$, $J'\subseteq J$ and
\[
\sum_{i\in I'}X_i
=
\sum_{s\in J'}Y_s.
\]
The pairs $(\varnothing,\varnothing)$ and $(I,J)$ are called the
\emph{empty} and \emph{full} subexchanges, respectively.  A subexchange is
\emph{proper} if it is different from the full subexchange and
\emph{nontrivial} if it is neither empty nor full.  The exchange is
\emph{irreducible} if it has no nontrivial subexchange, and
\emph{reducible} otherwise.
\end{definition}

The positivity assumption immediately rules out one-sided subexchanges.

\begin{remark}\label{rem:no-one-sided}
A nonempty subexchange cannot have exactly one empty side.  If, say,
$I'\neq\varnothing$ and $J'=\varnothing$, then
\[
\sum_{i\in I'}X_i=0,
\]
which is impossible because the entry sum of the left side is positive.
Consequently every nonempty subexchange has both sides nonempty.

The complement
\[
(I\setminus I',\,J\setminus J')
\]
of any subexchange, obtained by subtracting it from the full exchange, is
again a subexchange.  If the original subexchange is nonempty and proper,
then its complement is also nonempty and hence has both sides nonempty.
\end{remark}

Every exchange can be decomposed into irreducible pieces.

\begin{lemma}[Block decomposition]\label{lem:exchange-decomposition}
\leavevmode
\begin{enumerate}[(i)]
\item Every positive exchange decomposes as a disjoint union of
irreducible positive exchanges: there are partitions
\[
I=I_1\sqcup\cdots\sqcup I_m,
\qquad
J=J_1\sqcup\cdots\sqcup J_m,
\]
such that each pair $(I_k,J_k)$ is an irreducible positive exchange.
\item If the two sides of the exchange are minimal nonnegative
decompositions of the same tensor, then every block is \emph{balanced},
\[
|I_k|=|J_k|;
\]
in particular, the two decompositions have equal length.
\end{enumerate}
\end{lemma}

\begin{proof}
\emph{Part (i).}
Let $\mathcal S$ be the set of all subexchanges other than the empty
subexchange $(\varnothing,\varnothing)$.  This finite set is nonempty
because it contains the full subexchange $(I,J)$.  Order $\mathcal S$ by
componentwise inclusion and choose a minimal element $(I_1,J_1)$.  By
Remark~\ref{rem:no-one-sided}, both $I_1$ and $J_1$ are nonempty.

The exchange
\[
\sum_{i\in I_1}X_i
=
\sum_{s\in J_1}Y_s
\]
is irreducible.  Indeed, a nontrivial subexchange of this block would be a
nonempty subexchange of the original exchange that is strictly smaller
than $(I_1,J_1)$, contrary to the choice of $(I_1,J_1)$.

By Remark~\ref{rem:no-one-sided}, the complement
\[
(I\setminus I_1,\,J\setminus J_1)
\]
is again a subexchange.  If it is empty, we are done.  Otherwise we apply
the same argument to the complementary exchange.  At every step the total
number of remaining indices strictly decreases, so the process terminates
after finitely many steps.

\emph{Part (ii).}
Let
\[
\T=\sum_{i\in I}X_i=\sum_{s\in J}Y_s
\]
with both decompositions minimal, and let $(I_k,J_k)$ be any block from
part (i).  Suppose $|I_k|>|J_k|$.  Replacing the terms
$\{X_i:i\in I_k\}$ by $\{Y_s:s\in J_k\}$ leaves the sum unchanged, because
the block is a subexchange.  The resulting decomposition has length
\[
|I|-|I_k|+|J_k|<|I|,
\]
contradicting the minimality of the first decomposition.  The case
$|I_k|<|J_k|$ is symmetric, using the minimality of the second
decomposition.  Hence every block is balanced, and summing over the blocks
gives $|I|=|J|$.
\end{proof}

Thus an exchange witnessing failure of minimality must contain an
\emph{unbalanced} irreducible block, whereas an exchange between two
different minimal decompositions contains a \emph{balanced} irreducible
block of size at least two.  This distinction is what ultimately produces
the two thresholds in Theorem~\ref{thm:main}.

\subsection{Connectedness and the Lovitz--Petrov Splitting Theorem}

The next step is to encode a positive exchange as a signed family of
product tensors.  Irreducibility of the exchange will then translate into
connectedness of this signed family.

\begin{definition}[Splitting and connectedness]\label{def:connected}
Following \citet[Definition~3]{LovitzPetrov2023}, a finite multiset $E$ of
nonzero vectors in a real vector space \emph{splits} if it has a nonempty
proper submultiset $F$ such that
\[
\spann E
=
\spann F\oplus\spann(E\setminus F).
\]
We call $E$ \emph{connected} if it does not split.  A multiset with a
single element is connected.
\end{definition}

The following theorem is the linear-algebraic ingredient in our argument.

\begin{theorem}[Lovitz--Petrov splitting theorem]
\label{thm:lp-splitting}
Let
\[
E=
\{x_{a,1}\otimes\cdots\otimes x_{a,d}:a\in[N]\}
\]
be a finite multiset of nonzero product tensors over a field.  For each
mode put
\[
r_j
=
\dim\spann\{x_{a,j}:a\in[N]\}.
\]
If
\[
\dim\spann E
\le
\sum_{j=1}^{d}(r_j-1),
\]
then $E$ splits.
\end{theorem}

We will apply Theorem~\ref{thm:lp-splitting} only over $\R$, to signed
families obtained from irreducible positive exchanges.  The resulting
dimension bound is the linear-algebraic component of the positive splitting
inequality.  The complementary component, which has no analogue for
arbitrary signed exchanges, comes from the support geometry imposed by
nonnegativity and is developed in Sections~\ref{sec:cones}--\ref{sec:tau}.



\section{Support Geometry and Bridge Connectivity}\label{sec:cones}

The previous section developed the linear-algebraic component of the
identifiability argument through the Lovitz--Petrov splitting theorem.  We
now develop the complementary structure created by nonnegativity.  The basic
observation is that the nonnegative vectors in a factor space form an
intrinsic polyhedral cone.  Its facets record support information that is
invisible to ordinary linear dimension.  We use these facets to associate a
signature to each factor, then combine the signatures across modes into a
bridge graph.  The final result of the section shows that if this bridge
graph is disconnected, then a positive exchange must itself decompose into
smaller exchanges.

\subsection{Intrinsic Cones and Facet Signatures}

Let $W$ be a linear subspace of $\R^n$.  Call $W$ \emph{admissible} if it is
spanned by its nonnegative vectors, i.e.,
\[
W=\spann\bigl(W\cap\Rnn^n\bigr).
\]
For admissible $W$, define the \emph{intrinsic cone}
\[
C(W)=W\cap\Rnn^n.
\]
Throughout this section, $W$ is admissible with $\dim W\ge1$; this is the
only case needed below, since all factor spaces considered later contain
nonzero nonnegative vectors.

The following lemma identifies the facet structure of the intrinsic cone.
In particular, although $C(W)$ may lie in a lower-dimensional subspace of
$\R^n$, every facet is still exposed by one of the original coordinate
functionals.

\begin{lemma}[Structure of the intrinsic cone]\label{lem:intrinsic-cone}
Let $W\subseteq\R^n$ be admissible with $\dim W\ge1$. Then:
\begin{enumerate}[(i)]
\item $C(W)$ is a polyhedral cone that is pointed and full-dimensional in $W$;
\item every facet $F$ of $C(W)$ has the form
\[
F=C(W)\cap\{x\in\R^n:x_i=0\}
\]
for at least one coordinate $i$ whose functional $x\mapsto x_i$ is not
identically zero on $W$;
\item if a facet $F$ satisfies
\[
F=C(W)\cap\{x_i=0\}=C(W)\cap\{x_k=0\}
\]
for two coordinates $i,k$, then the restrictions of $x_i$ and $x_k$ to
$W$ are positive multiples of one another;
\item choosing for every facet $F$ one coordinate $i(F)$ as in
\textup{(ii)}, one has
\[
C(W)=
\bigl\{
x\in W:
x_{i(F)}\ge0
\ \text{for every facet }F\text{ of }C(W)
\bigr\}.
\]
\end{enumerate}
\end{lemma}

\begin{proof}
\emph{(i).}
The cone $C(W)$ is the intersection of the subspace $W$ with the finitely
many closed halfspaces
\[
\{x\in\R^n:x_i\ge0\},
\qquad i\in[n],
\]
and is therefore polyhedral.  It is pointed because
\[
C(W)\cap(-C(W))
\subseteq
\Rnn^n\cap(-\Rnn^n)
=\{0\}.
\]
It is full-dimensional in $W$ because admissibility gives
\[
\spann C(W)=W.
\]

\emph{(ii).}
Let
\[
Z=\{i\in[n]:x_i=0\text{ for all }x\in W\}
\]
be the set of coordinates that vanish identically on $W$.  Let $F$ be a
facet of $C(W)$ and choose $z\in\relint F$.  We claim that $z_i=0$ for some
$i\notin Z$.  Otherwise $z_i>0$ for every $i\notin Z$.  Since there are only
finitely many such coordinates, all these inequalities remain strict in a
neighborhood of $z$ in $W$, while the coordinates in $Z$ vanish identically
on $W$.  Hence $z\in\relint C(W)$, contradicting the fact that a point in the
relative interior of a proper face cannot lie in the relative interior of
the full-dimensional polyhedron
\citep[Theorem~6.2, Corollary~18.1.3]{Rockafellar1970}.

Thus choose $i\notin Z$ with $z_i=0$ and set
\[
G=C(W)\cap\{x_i=0\}.
\]
The functional $x\mapsto x_i$ is nonnegative on $C(W)$, so $G$ is an
exposed face.  Since $i\notin Z$, this functional is not identically zero
on $W$, hence it is positive at some point of $C(W)$ and therefore
$G\neq C(W)$.

We next show that $F\subseteq G$.  The functional
$\varphi(x)=x_i$ is nonnegative on $C(W)$ and vanishes at
$z\in\relint F$.  For any $y\in F$, the relative interior property implies
that the segment from $y$ to $z$ can be extended beyond $z$ while remaining
in $F$ \citep[Theorem~6.4]{Rockafellar1970}.  Hence there exist
$y'\in F$ and $\mu\in(0,1)$ such that
\[
z=\mu y+(1-\mu)y'.
\]
Since $\varphi(y),\varphi(y')\ge0$ and $\varphi(z)=0$, we obtain
$\varphi(y)=0$, so $y\in G$.

Finally, $F=G$.  Since $F$ is a facet,
\[
\dim F=\dim W-1.
\]
Moreover $F\subseteq G\subsetneq C(W)$, so
$\dim G\le\dim W-1$.  Since $F\subseteq G$,
$\dim G\ge\dim F$, and therefore $\dim G=\dim F$.  Thus
\[
F=G=C(W)\cap\{x_i=0\}.
\]

\emph{(iii).}
Let $H=\spann F$.  Since $F$ is a facet of the full-dimensional cone
$C(W)$, $H$ is a hyperplane of $W$.  The restrictions
\[
\varphi_i=x_i|_W,
\qquad
\varphi_k=x_k|_W
\]
both vanish on $F$ and hence on $H$.  Neither is identically zero on $W$,
because each corresponding coordinate defines a proper face.  Their
kernels therefore equal the hyperplane $H$, so the two functionals are
proportional:
\[
\varphi_k=\lambda\varphi_i
\]
for some $\lambda\neq0$.  Choose $c\in C(W)$ with
$\varphi_i(c)>0$.  Since $\varphi_k(c)\ge0$, we obtain $\lambda>0$.

\emph{(iv).}
The inclusion ``$\subseteq$'' is immediate.  For the converse, observe
that $C(W)$ is a pointed, full-dimensional polyhedral cone in $W$.  Hence
it is the intersection of the halfspaces defined by its facets
\citep[Theorem~8.1]{Schrijver1986}.  By parts (ii)--(iii), each facet
halfspace can be represented by a coordinate functional
$x\mapsto x_{i(F)}$ restricted to $W$.  Thus
\[
C(W)=
\bigl\{
x\in W:
x_{i(F)}\ge0
\ \text{for every facet }F
\bigr\}.
\]
\end{proof}

The preceding lemma shows that the facet structure of $C(W)$ can be
represented entirely by coordinate functionals.  Although several
coordinates may define the same facet, Lemma~\ref{lem:intrinsic-cone}(iii)
shows that such representatives differ only by positive scaling.  This
allows us to record which facets are strictly positive for a given vector.

Let $\mathcal F(W)$ denote the finite set of facets of $C(W)$ and choose,
for each $F\in\mathcal F(W)$, one coordinate representative $i(F)$.  Define
the \emph{facet-coordinate map}
\[
L_W:W\longrightarrow\R^{\mathcal F(W)},
\qquad
L_Wx=\bigl(x_{i(F)}\bigr)_{F\in\mathcal F(W)}.
\]
By Lemma~\ref{lem:intrinsic-cone}(iii), changing the representative of a
facet only rescales the corresponding coordinate by a positive constant.

\begin{lemma}\label{lem:facet-map}
The map $L_W$ is linear and injective, and it maps $C(W)$ into the
nonnegative orthant $\Rnn^{\mathcal F(W)}$.  In particular, every nonzero
$x\in C(W)$ has a nonempty \emph{positive facet signature}
\begin{equation}
\mathcal A_W(x)
=
\{F\in\mathcal F(W):x_{i(F)}>0\}
=
\supp(L_Wx)
\neq\varnothing,
\label{eq:signature}
\end{equation}
and
\[
\mathcal A_W(\lambda x)=\mathcal A_W(x)
\qquad\text{for every }\lambda>0.
\]
\end{lemma}

\begin{proof}
Linearity is immediate, and $L_Wx$ is entrywise nonnegative whenever
$x\in C(W)$.  Suppose $L_Wx=0$ for some $x\in W$.  Then
$x_{i(F)}=0$ for every facet $F$.  By Lemma~\ref{lem:intrinsic-cone}(iv),
both $x$ and $-x$ satisfy all facet inequalities, so
\[
x\in C(W)\cap(-C(W))=\{0\}.
\]
Thus $L_W$ is injective.  In particular, $L_Wx\neq0$ whenever $x\neq0$,
which proves that every nonzero $x\in C(W)$ has a nonempty signature.
Invariance under positive rescaling is immediate.
\end{proof}

The signature $\mathcal A_W(x)$ records the facets on which $x$ has a
strictly positive coordinate representative.  Two nonzero vectors in the
same intrinsic cone may nevertheless have disjoint signatures; Figure~\ref{fig:cone}
illustrates this phenomenon.  This observation motivates the notion of
facet-opposite factors used below.

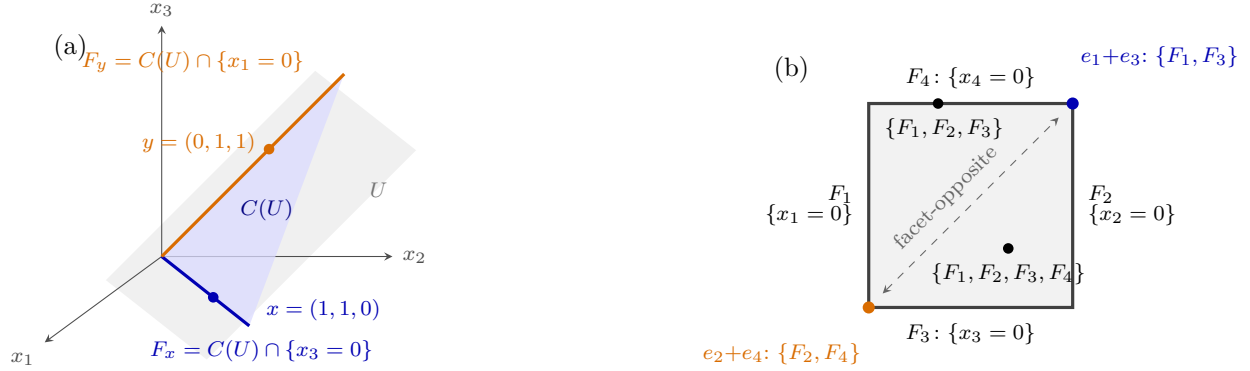
\begin{figure*}[htbp]
\centering
\begin{tikzpicture}[scale=1.42,>=stealth]
\coordinate (O)  at (0,0);
\coordinate (X)  at (0.48,-0.38);   
\coordinate (Y)  at (1,1);          
\coordinate (Xf) at (0.816,-0.646); 
\coordinate (Yf) at (1.7,1.7);      
\fill[black!6]
(-0.52,-0.22) -- (0.42,-0.96) -- (2.37,0.99) -- (1.43,1.73) -- cycle;
\node[black!55,font=\scriptsize] at (2.02,0.62) {$U$};
\draw[->,black!70] (O) -- (-1.09,-0.80) node[below left,font=\scriptsize] {$x_1$};
\draw[->,black!70] (O) -- (2.15,0) node[right,font=\scriptsize] {$x_2$};
\draw[->,black!70] (O) -- (0,2.15) node[above,font=\scriptsize] {$x_3$};
\fill[blue!14,opacity=0.75] (O) -- (Xf) -- (Yf) -- cycle;
\node[blue!50!black,font=\scriptsize] at (0.98,0.44) {$C(U)$};
\draw[very thick,blue!70!black]  (O) -- (Xf);
\draw[very thick,orange!85!black] (O) -- (Yf);
\node[blue!70!black,font=\scriptsize,below,align=center] at (0.93,-0.68)
{$F_x=C(U)\cap\{x_3=0\}$};
\node[orange!85!black,font=\scriptsize,above left,align=center] at (1.42,1.62)
{$F_y=C(U)\cap\{x_1=0\}$};
\fill[blue!70!black]  (X) circle (1.4pt);
\node[blue!70!black,font=\scriptsize,right] at (0.88,-0.50) {$x=(1,1,0)$};
\fill[orange!85!black] (Y) circle (1.4pt);
\node[orange!85!black,font=\scriptsize,left] at (0.97,1.06) {$y=(0,1,1)$};
\node[font=\small] at (-0.85,1.95) {(a)};
\end{tikzpicture}%
\hfill
\begin{tikzpicture}[scale=1.42,>=stealth]
\def\L{1.9}
\fill[black!5] (0,0) rectangle (\L,\L);
\draw[very thick,black!75] (0,0) rectangle (\L,\L);
\node[left,font=\scriptsize,align=right]  at (-0.05,0.5*\L) {$F_1$\\[-2pt]$\{x_1=0\}$};
\node[right,font=\scriptsize,align=left]  at (\L+0.05,0.5*\L) {$F_2$\\[-2pt]$\{x_2=0\}$};
\node[below,font=\scriptsize] at (0.5*\L,-0.05) {$F_3\colon\{x_3=0\}$};
\node[above,font=\scriptsize] at (0.5*\L,\L+0.05) {$F_4\colon\{x_4=0\}$};
\draw[dashed,black!60,<->] (0.13,0.13) -- (\L-0.13,\L-0.13);
\node[rotate=45,font=\scriptsize,black!60] at (0.72,1.02)
{facet-opposite};
\fill[blue!70!black] (\L,\L) circle (1.6pt);
\node[above right,font=\scriptsize,blue!70!black,align=left] at (\L-0.02,\L+0.24)
{$e_1{+}e_3$:\ $\{F_1,F_3\}$};
\fill[orange!85!black] (0,0) circle (1.6pt);
\node[below left,font=\scriptsize,orange!85!black,align=left] at (0.02,-0.24)
{$e_2{+}e_4$:\ $\{F_2,F_4\}$};
\fill[black] (1.30,0.55) circle (1.4pt);
\node[below,font=\scriptsize] at (1.30,0.51) {$\{F_1,F_2,F_3,F_4\}$};
\fill[black] (0.34*\L,\L) circle (1.4pt);
\node[below,font=\scriptsize] at (0.37*\L,\L-0.04) {$\{F_1,F_2,F_3\}$};
\node[font=\small] at (-0.72,\L+0.32) {(b)};
\end{tikzpicture}
\caption{Intrinsic cones, facets and signatures.
(a) The intrinsic cone $C(U)=U\cap\Rnn^3$ of the admissible plane
$U=\spann\{x,y\}$ with $x=(1,1,0)$, $y=(0,1,1)$: a pointed two-dimensional
cone whose facets are its extreme rays, each cut out by a coordinate as in
Lemma~\ref{lem:intrinsic-cone}(ii).  Each generator, $x$ and $y$, lies on its own
facet, so its coordinate $i(F)$ vanishes there and the signatures are
$\mathcal A_U(x)=\{F_y\}$ and $\mathcal A_U(y)=\{F_x\}$, which are disjoint.
(b) The cross-section $\{x_1+x_2=x_3+x_4=1\}$ of the three-dimensional
intrinsic cone of the admissible space
$W=\{x\in\R^4:x_1+x_2=x_3+x_4\}$: a cone over a square with the four
facets $F_i=C(W)\cap\{x_i=0\}$, $i(F_i)=i$.  Sample points are labeled by
their signatures $\mathcal A_W(\cdot)$: full in the relative interior,
smaller on proper faces.  The two marked corners have disjoint signatures
although they lie in one and the same intrinsic cone.}
\label{fig:cone}
\end{figure*}


\subsection{Bridge Graphs and Rectangular Splitting}\label{sec:bridge}

We now combine facet signatures across modes.  Fix a finite index set $S$ and,
for each mode $j$, a family of nonzero nonnegative vectors
\[
x_r^{(j)}\in\Rnn^{n_j},
\qquad r\in S,
\]
together with an admissible subspace
$W_j\subseteq\R^{n_j}$ containing all of them.  Write
\[
\mathcal A_j(r)
=
\mathcal A_{W_j}\bigl(x_r^{(j)}\bigr)
\subseteq\mathcal F(W_j),
\qquad r\in S.
\]
These signatures are nonempty by Lemma~\ref{lem:facet-map}.

\begin{definition}[Bridge graph]\label{def:bridge}
Two indices $r,s\in S$ are \emph{facet-opposite in mode $j$} with respect
to $W_j$ if
\[
\mathcal A_j(r)\cap\mathcal A_j(s)=\varnothing.
\]
The \emph{bridge graph}
\[
\Gamma=\Gamma\bigl((W_j)_j\bigr)
\]
has vertex set $S$ and an edge $\{r,s\}$ whenever $r\neq s$ are
facet-opposite in at most one mode.
\end{definition}

Thus an edge of the bridge graph means that the two components have
compatible facet signatures in all but possibly one mode.  This definition
is chosen precisely so that bridge edges correspond to one-coordinate
moves in the product of the facet sets.

For each $r\in S$, define the \emph{facet box}
\[
B_r
=
\mathcal A_1(r)\times\cdots\times\mathcal A_d(r)
\subseteq
\mathcal F(W_1)\times\cdots\times\mathcal F(W_d)
=:\Omega.
\]
The box is nonempty because every factor signature is nonempty.  On
$\Omega$, call two points \emph{rook-adjacent} if they differ in at most
one coordinate.  A subset of $\Omega$ is \emph{rook-connected} if every
two of its points can be joined by a path of rook-adjacent points lying
inside the subset.  Every box $B_r$ is rook-connected, since its
coordinates may be changed one at a time without leaving the box.

\begin{lemma}[Bridges and rooks]\label{lem:bridge-rook}
Let
\[
U=\bigcup_{t\in S}B_t\subseteq\Omega.
\]
Two indices $r,s\in S$ lie in the same connected component of the bridge
graph $\Gamma$ if and only if $B_r$ and $B_s$ lie in the same
rook-connected component of $U$.  Consequently, the assignment
\[
r\longmapsto
\text{the rook component of $U$ containing $B_r$}
\]
induces a bijection between the connected components of $\Gamma$ and the
rook components of $U$.
\end{lemma}

\begin{proof}
Each box $B_t$ is rook-connected and hence is contained in a single rook
component of $U$.  The resulting assignment from bridge-graph vertices to
rook components is therefore well defined and surjective.

It remains to prove that two vertices belong to the same component on one
side if and only if their boxes belong to the same component on the other.

\emph{Bridge path $\Rightarrow$ same rook component.}
It suffices to consider a single bridge edge $\{r,s\}$.  If the signatures
intersect in every mode, choose
\[
f_j\in\mathcal A_j(r)\cap\mathcal A_j(s)
\qquad\text{for all }j.
\]
Then
$(f_1,\dots,f_d)\in B_r\cap B_s$, so the two boxes lie in the same rook
component.

Otherwise there is a unique exceptional mode $j_0$ in which the
signatures are disjoint.  For every $j\neq j_0$, choose
\[
f_j\in\mathcal A_j(r)\cap\mathcal A_j(s),
\]
and choose arbitrary
\[
f_{j_0}\in\mathcal A_{j_0}(r),
\qquad
g_{j_0}\in\mathcal A_{j_0}(s).
\]
Then
\[
u=(f_1,\dots,f_d)\in B_r,
\]
while
\[
v=
(f_1,\dots,f_{j_0-1},g_{j_0},f_{j_0+1},\dots,f_d)\in B_s.
\]
The two points differ only in coordinate $j_0$, so they are rook-adjacent.
Since both boxes are rook-connected, they belong to the same rook component.
Concatenating along a bridge path proves the implication.

\emph{Same rook component $\Rightarrow$ bridge path.}
Let
\[
u_0,u_1,\dots,u_m
\]
be a rook path in $U$ with $u_0\in B_r$ and $u_m\in B_s$.  For each $k$
choose $t_k\in S$ such that $u_k\in B_{t_k}$, taking
$t_0=r$ and $t_m=s$.

For a fixed $k$, the points $u_k$ and $u_{k+1}$ agree in every coordinate
except possibly one, say mode $j_0$.  Therefore, for every $j\neq j_0$,
\[
(u_k)_j=(u_{k+1})_j
\in
\mathcal A_j(t_k)\cap\mathcal A_j(t_{k+1}).
\]
Thus $t_k$ and $t_{k+1}$ are facet-opposite in at most one mode.  Hence
either $t_k=t_{k+1}$ or $\{t_k,t_{k+1}\}$ is a bridge edge.  The sequence
$t_0,\dots,t_m$ therefore gives a walk from $r$ to $s$ in $\Gamma$.
\end{proof}

Figure~\ref{fig:rook} gives a three-mode illustration.  The key point is
that bridge edges correspond exactly to rook-compatible transitions
between facet boxes, whereas a disconnected bridge graph causes the union
of the boxes to separate into distinct rook components.

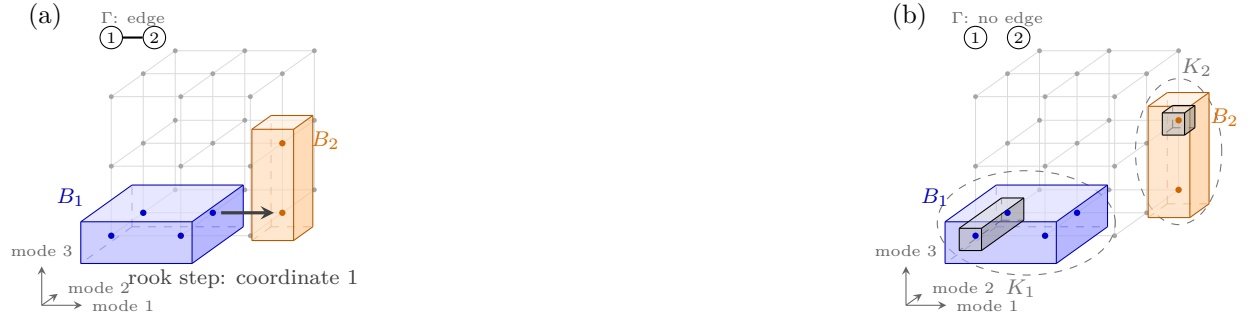
\begin{figure*}[htbp]
\centering
\begin{tikzpicture}[scale=0.92,>=stealth]
\draw[black!14,very thin] (0.000,0.000) -- (2.000,0.000);
\draw[black!14,very thin] (0.000,1.000) -- (2.000,1.000);
\draw[black!14,very thin] (0.000,2.000) -- (2.000,2.000);
\draw[black!14,very thin] (0.460,0.330) -- (2.460,0.330);
\draw[black!14,very thin] (0.460,1.330) -- (2.460,1.330);
\draw[black!14,very thin] (0.460,2.330) -- (2.460,2.330);
\draw[black!14,very thin] (0.920,0.660) -- (2.920,0.660);
\draw[black!14,very thin] (0.920,1.660) -- (2.920,1.660);
\draw[black!14,very thin] (0.920,2.660) -- (2.920,2.660);
\draw[black!14,very thin] (0.000,0.000) -- (0.920,0.660);
\draw[black!14,very thin] (0.000,1.000) -- (0.920,1.660);
\draw[black!14,very thin] (0.000,2.000) -- (0.920,2.660);
\draw[black!14,very thin] (1.000,0.000) -- (1.920,0.660);
\draw[black!14,very thin] (1.000,1.000) -- (1.920,1.660);
\draw[black!14,very thin] (1.000,2.000) -- (1.920,2.660);
\draw[black!14,very thin] (2.000,0.000) -- (2.920,0.660);
\draw[black!14,very thin] (2.000,1.000) -- (2.920,1.660);
\draw[black!14,very thin] (2.000,2.000) -- (2.920,2.660);
\draw[black!14,very thin] (0.000,0.000) -- (0.000,2.000);
\draw[black!14,very thin] (0.460,0.330) -- (0.460,2.330);
\draw[black!14,very thin] (0.920,0.660) -- (0.920,2.660);
\draw[black!14,very thin] (1.000,0.000) -- (1.000,2.000);
\draw[black!14,very thin] (1.460,0.330) -- (1.460,2.330);
\draw[black!14,very thin] (1.920,0.660) -- (1.920,2.660);
\draw[black!14,very thin] (2.000,0.000) -- (2.000,2.000);
\draw[black!14,very thin] (2.460,0.330) -- (2.460,2.330);
\draw[black!14,very thin] (2.920,0.660) -- (2.920,2.660);
\fill[black!35] (0.000,1.000) circle (1.0pt);
\fill[black!35] (0.000,2.000) circle (1.0pt);
\fill[black!35] (0.460,1.330) circle (1.0pt);
\fill[black!35] (0.460,2.330) circle (1.0pt);
\fill[black!35] (0.920,0.660) circle (1.0pt);
\fill[black!35] (0.920,1.660) circle (1.0pt);
\fill[black!35] (0.920,2.660) circle (1.0pt);
\fill[black!35] (1.000,1.000) circle (1.0pt);
\fill[black!35] (1.000,2.000) circle (1.0pt);
\fill[black!35] (1.460,1.330) circle (1.0pt);
\fill[black!35] (1.460,2.330) circle (1.0pt);
\fill[black!35] (1.920,0.660) circle (1.0pt);
\fill[black!35] (1.920,1.660) circle (1.0pt);
\fill[black!35] (1.920,2.660) circle (1.0pt);
\fill[black!35] (2.000,0.000) circle (1.0pt);
\fill[black!35] (2.000,1.000) circle (1.0pt);
\fill[black!35] (2.000,2.000) circle (1.0pt);
\fill[black!35] (2.460,2.330) circle (1.0pt);
\fill[black!35] (2.920,0.660) circle (1.0pt);
\fill[black!35] (2.920,1.660) circle (1.0pt);
\fill[black!35] (2.920,2.660) circle (1.0pt);
\draw[blue!70!black,thin,dashed,opacity=0.55] (0.298,0.129) -- (1.898,0.129);
\draw[blue!70!black,thin,dashed,opacity=0.55] (0.298,0.129) -- (-0.438,-0.399);
\draw[blue!70!black,thin,dashed,opacity=0.55] (0.298,0.129) -- (0.298,0.729);
\fill[blue!30,fill opacity=0.62] (-0.438,-0.399) -- (1.162,-0.399) -- (1.162,0.201) -- (-0.438,0.201) -- cycle;
\fill[blue!15,fill opacity=0.62] (-0.438,0.201) -- (1.162,0.201) -- (1.898,0.729) -- (0.298,0.729) -- cycle;
\fill[blue!45,fill opacity=0.62] (1.162,-0.399) -- (1.898,0.129) -- (1.898,0.729) -- (1.162,0.201) -- cycle;
\draw[blue!70!black,thin] (-0.438,-0.399) -- (1.162,-0.399);
\draw[blue!70!black,thin] (1.162,-0.399) -- (1.162,0.201);
\draw[blue!70!black,thin] (1.162,0.201) -- (-0.438,0.201);
\draw[blue!70!black,thin] (-0.438,0.201) -- (-0.438,-0.399);
\draw[blue!70!black,thin] (-0.438,0.201) -- (0.298,0.729);
\draw[blue!70!black,thin] (1.162,0.201) -- (1.898,0.729);
\draw[blue!70!black,thin] (0.298,0.729) -- (1.898,0.729);
\draw[blue!70!black,thin] (1.162,-0.399) -- (1.898,0.129);
\draw[blue!70!black,thin] (1.898,0.129) -- (1.898,0.729);
\draw[orange!70!black,thin,dashed,opacity=0.55] (2.298,0.129) -- (2.898,0.129);
\draw[orange!70!black,thin,dashed,opacity=0.55] (2.298,0.129) -- (2.022,-0.069);
\draw[orange!70!black,thin,dashed,opacity=0.55] (2.298,0.129) -- (2.298,1.729);
\fill[orange!30,fill opacity=0.62] (2.022,-0.069) -- (2.622,-0.069) -- (2.622,1.531) -- (2.022,1.531) -- cycle;
\fill[orange!15,fill opacity=0.62] (2.022,1.531) -- (2.622,1.531) -- (2.898,1.729) -- (2.298,1.729) -- cycle;
\fill[orange!45,fill opacity=0.62] (2.622,-0.069) -- (2.898,0.129) -- (2.898,1.729) -- (2.622,1.531) -- cycle;
\draw[orange!70!black,thin] (2.022,-0.069) -- (2.622,-0.069);
\draw[orange!70!black,thin] (2.622,-0.069) -- (2.622,1.531);
\draw[orange!70!black,thin] (2.622,1.531) -- (2.022,1.531);
\draw[orange!70!black,thin] (2.022,1.531) -- (2.022,-0.069);
\draw[orange!70!black,thin] (2.022,1.531) -- (2.298,1.729);
\draw[orange!70!black,thin] (2.622,1.531) -- (2.898,1.729);
\draw[orange!70!black,thin] (2.298,1.729) -- (2.898,1.729);
\draw[orange!70!black,thin] (2.622,-0.069) -- (2.898,0.129);
\draw[orange!70!black,thin] (2.898,0.129) -- (2.898,1.729);
\fill[blue!80!black] (0.460,0.330) circle (1.3pt);
\fill[blue!80!black] (1.000,0.000) circle (1.3pt);
\fill[blue!80!black] (0.000,0.000) circle (1.3pt);
\fill[blue!80!black] (1.460,0.330) circle (1.3pt);
\fill[orange!80!black] (2.460,0.330) circle (1.3pt);
\fill[orange!80!black] (2.460,1.330) circle (1.3pt);
\draw[->,very thick,black!75] (1.580,0.330) -- (2.340,0.330);
\node[black!75,font=\scriptsize] at (1.90,-0.62) {rook step: coordinate $1$};
\node[blue!70!black,font=\scriptsize] at (-0.580,0.550) {$B_1$};
\node[orange!85!black,font=\scriptsize] at (3.080,1.380) {$B_2$};
\draw[->,black!60] (-1.00,-1.00) -- (-0.42,-1.00) node[right,font=\tiny,black!60] {mode 1};
\draw[->,black!60] (-1.00,-1.00) -- (-0.77,-0.835) node[right,font=\tiny,black!60,yshift=2.5pt] {mode 2};
\draw[->,black!60] (-1.00,-1.00) -- (-1.00,-0.44) node[above,font=\tiny,black!60] {mode 3};
\node[circle,draw,inner sep=1.1pt,font=\tiny] (g1) at (0.0,2.85) {1};
\node[circle,draw,inner sep=1.1pt,font=\tiny] (g2) at (0.62,2.85) {2};
\draw[thick] (g1) -- (g2);
\node[font=\tiny,black!60] at (0.31,3.12) {$\Gamma$: edge};
\node[font=\small] at (-0.95,3.15) {(a)};
\end{tikzpicture}
\hfill
\begin{tikzpicture}[scale=0.92,>=stealth]
\draw[black!14,very thin] (0.000,0.000) -- (2.000,0.000);
\draw[black!14,very thin] (0.000,1.000) -- (2.000,1.000);
\draw[black!14,very thin] (0.000,2.000) -- (2.000,2.000);
\draw[black!14,very thin] (0.460,0.330) -- (2.460,0.330);
\draw[black!14,very thin] (0.460,1.330) -- (2.460,1.330);
\draw[black!14,very thin] (0.460,2.330) -- (2.460,2.330);
\draw[black!14,very thin] (0.920,0.660) -- (2.920,0.660);
\draw[black!14,very thin] (0.920,1.660) -- (2.920,1.660);
\draw[black!14,very thin] (0.920,2.660) -- (2.920,2.660);
\draw[black!14,very thin] (0.000,0.000) -- (0.920,0.660);
\draw[black!14,very thin] (0.000,1.000) -- (0.920,1.660);
\draw[black!14,very thin] (0.000,2.000) -- (0.920,2.660);
\draw[black!14,very thin] (1.000,0.000) -- (1.920,0.660);
\draw[black!14,very thin] (1.000,1.000) -- (1.920,1.660);
\draw[black!14,very thin] (1.000,2.000) -- (1.920,2.660);
\draw[black!14,very thin] (2.000,0.000) -- (2.920,0.660);
\draw[black!14,very thin] (2.000,1.000) -- (2.920,1.660);
\draw[black!14,very thin] (2.000,2.000) -- (2.920,2.660);
\draw[black!14,very thin] (0.000,0.000) -- (0.000,2.000);
\draw[black!14,very thin] (0.460,0.330) -- (0.460,2.330);
\draw[black!14,very thin] (0.920,0.660) -- (0.920,2.660);
\draw[black!14,very thin] (1.000,0.000) -- (1.000,2.000);
\draw[black!14,very thin] (1.460,0.330) -- (1.460,2.330);
\draw[black!14,very thin] (1.920,0.660) -- (1.920,2.660);
\draw[black!14,very thin] (2.000,0.000) -- (2.000,2.000);
\draw[black!14,very thin] (2.460,0.330) -- (2.460,2.330);
\draw[black!14,very thin] (2.920,0.660) -- (2.920,2.660);
\fill[black!35] (0.000,1.000) circle (1.0pt);
\fill[black!35] (0.000,2.000) circle (1.0pt);
\fill[black!35] (0.460,1.330) circle (1.0pt);
\fill[black!35] (0.460,2.330) circle (1.0pt);
\fill[black!35] (0.920,0.660) circle (1.0pt);
\fill[black!35] (0.920,1.660) circle (1.0pt);
\fill[black!35] (0.920,2.660) circle (1.0pt);
\fill[black!35] (1.000,1.000) circle (1.0pt);
\fill[black!35] (1.000,2.000) circle (1.0pt);
\fill[black!35] (1.460,1.330) circle (1.0pt);
\fill[black!35] (1.460,2.330) circle (1.0pt);
\fill[black!35] (1.920,0.660) circle (1.0pt);
\fill[black!35] (1.920,1.660) circle (1.0pt);
\fill[black!35] (1.920,2.660) circle (1.0pt);
\fill[black!35] (2.000,0.000) circle (1.0pt);
\fill[black!35] (2.000,1.000) circle (1.0pt);
\fill[black!35] (2.000,2.000) circle (1.0pt);
\fill[black!35] (2.460,0.330) circle (1.0pt);
\fill[black!35] (2.460,1.330) circle (1.0pt);
\fill[black!35] (2.460,2.330) circle (1.0pt);
\fill[black!35] (2.920,2.660) circle (1.0pt);
\draw[dashed,black!55] (0.730,0.185) ellipse (1.28 and 0.75);
\draw[dashed,black!55] (2.920,1.210) ellipse (0.62 and 1.05);
\node[black!55,font=\scriptsize] at (0.650,-0.760) {$K_1$};
\node[black!55,font=\scriptsize] at (3.180,2.420) {$K_2$};
\draw[blue!70!black,thin,dashed,opacity=0.55] (0.298,0.129) -- (1.898,0.129);
\draw[blue!70!black,thin,dashed,opacity=0.55] (0.298,0.129) -- (-0.438,-0.399);
\draw[blue!70!black,thin,dashed,opacity=0.55] (0.298,0.129) -- (0.298,0.729);
\fill[blue!30,fill opacity=0.62] (-0.438,-0.399) -- (1.162,-0.399) -- (1.162,0.201) -- (-0.438,0.201) -- cycle;
\fill[blue!15,fill opacity=0.62] (-0.438,0.201) -- (1.162,0.201) -- (1.898,0.729) -- (0.298,0.729) -- cycle;
\fill[blue!45,fill opacity=0.62] (1.162,-0.399) -- (1.898,0.129) -- (1.898,0.729) -- (1.162,0.201) -- cycle;
\draw[blue!70!black,thin] (-0.438,-0.399) -- (1.162,-0.399);
\draw[blue!70!black,thin] (1.162,-0.399) -- (1.162,0.201);
\draw[blue!70!black,thin] (1.162,0.201) -- (-0.438,0.201);
\draw[blue!70!black,thin] (-0.438,0.201) -- (-0.438,-0.399);
\draw[blue!70!black,thin] (-0.438,0.201) -- (0.298,0.729);
\draw[blue!70!black,thin] (1.162,0.201) -- (1.898,0.729);
\draw[blue!70!black,thin] (0.298,0.729) -- (1.898,0.729);
\draw[blue!70!black,thin] (1.162,-0.399) -- (1.898,0.129);
\draw[blue!70!black,thin] (1.898,0.129) -- (1.898,0.729);
\draw[orange!70!black,thin,dashed,opacity=0.55] (2.758,0.459) -- (3.358,0.459);
\draw[orange!70!black,thin,dashed,opacity=0.55] (2.758,0.459) -- (2.482,0.261);
\draw[orange!70!black,thin,dashed,opacity=0.55] (2.758,0.459) -- (2.758,2.059);
\fill[orange!30,fill opacity=0.62] (2.482,0.261) -- (3.082,0.261) -- (3.082,1.861) -- (2.482,1.861) -- cycle;
\fill[orange!15,fill opacity=0.62] (2.482,1.861) -- (3.082,1.861) -- (3.358,2.059) -- (2.758,2.059) -- cycle;
\fill[orange!45,fill opacity=0.62] (3.082,0.261) -- (3.358,0.459) -- (3.358,2.059) -- (3.082,1.861) -- cycle;
\draw[orange!70!black,thin] (2.482,0.261) -- (3.082,0.261);
\draw[orange!70!black,thin] (3.082,0.261) -- (3.082,1.861);
\draw[orange!70!black,thin] (3.082,1.861) -- (2.482,1.861);
\draw[orange!70!black,thin] (2.482,1.861) -- (2.482,0.261);
\draw[orange!70!black,thin] (2.482,1.861) -- (2.758,2.059);
\draw[orange!70!black,thin] (3.082,1.861) -- (3.358,2.059);
\draw[orange!70!black,thin] (2.758,2.059) -- (3.358,2.059);
\draw[orange!70!black,thin] (3.082,0.261) -- (3.358,0.459);
\draw[orange!70!black,thin] (3.358,0.459) -- (3.358,2.059);
\draw[black!70!black,thin,dashed,opacity=0.55] (0.374,0.223) -- (0.694,0.223);
\draw[black!70!black,thin,dashed,opacity=0.55] (0.374,0.223) -- (-0.234,-0.213);
\draw[black!70!black,thin,dashed,opacity=0.55] (0.374,0.223) -- (0.374,0.543);
\fill[black!30,fill opacity=0.5] (-0.234,-0.213) -- (0.086,-0.213) -- (0.086,0.107) -- (-0.234,0.107) -- cycle;
\fill[black!15,fill opacity=0.5] (-0.234,0.107) -- (0.086,0.107) -- (0.694,0.543) -- (0.374,0.543) -- cycle;
\fill[black!45,fill opacity=0.5] (0.086,-0.213) -- (0.694,0.223) -- (0.694,0.543) -- (0.086,0.107) -- cycle;
\draw[black!70!black,thin] (-0.234,-0.213) -- (0.086,-0.213);
\draw[black!70!black,thin] (0.086,-0.213) -- (0.086,0.107);
\draw[black!70!black,thin] (0.086,0.107) -- (-0.234,0.107);
\draw[black!70!black,thin] (-0.234,0.107) -- (-0.234,-0.213);
\draw[black!70!black,thin] (-0.234,0.107) -- (0.374,0.543);
\draw[black!70!black,thin] (0.086,0.107) -- (0.694,0.543);
\draw[black!70!black,thin] (0.374,0.543) -- (0.694,0.543);
\draw[black!70!black,thin] (0.086,-0.213) -- (0.694,0.223);
\draw[black!70!black,thin] (0.694,0.223) -- (0.694,0.543);
\draw[black!70!black,thin,dashed,opacity=0.55] (2.834,1.553) -- (3.154,1.553);
\draw[black!70!black,thin,dashed,opacity=0.55] (2.834,1.553) -- (2.686,1.447);
\draw[black!70!black,thin,dashed,opacity=0.55] (2.834,1.553) -- (2.834,1.873);
\fill[black!30,fill opacity=0.5] (2.686,1.447) -- (3.006,1.447) -- (3.006,1.767) -- (2.686,1.767) -- cycle;
\fill[black!15,fill opacity=0.5] (2.686,1.767) -- (3.006,1.767) -- (3.154,1.873) -- (2.834,1.873) -- cycle;
\fill[black!45,fill opacity=0.5] (3.006,1.447) -- (3.154,1.553) -- (3.154,1.873) -- (3.006,1.767) -- cycle;
\draw[black!70!black,thin] (2.686,1.447) -- (3.006,1.447);
\draw[black!70!black,thin] (3.006,1.447) -- (3.006,1.767);
\draw[black!70!black,thin] (3.006,1.767) -- (2.686,1.767);
\draw[black!70!black,thin] (2.686,1.767) -- (2.686,1.447);
\draw[black!70!black,thin] (2.686,1.767) -- (2.834,1.873);
\draw[black!70!black,thin] (3.006,1.767) -- (3.154,1.873);
\draw[black!70!black,thin] (2.834,1.873) -- (3.154,1.873);
\draw[black!70!black,thin] (3.006,1.447) -- (3.154,1.553);
\draw[black!70!black,thin] (3.154,1.553) -- (3.154,1.873);
\fill[blue!80!black] (0.460,0.330) circle (1.3pt);
\fill[blue!80!black] (1.000,0.000) circle (1.3pt);
\fill[blue!80!black] (0.000,0.000) circle (1.3pt);
\fill[blue!80!black] (1.460,0.330) circle (1.3pt);
\fill[orange!80!black] (2.920,0.660) circle (1.3pt);
\fill[orange!80!black] (2.920,1.660) circle (1.3pt);
\node[blue!70!black,font=\scriptsize] at (-0.580,0.550) {$B_1$};
\node[orange!85!black,font=\scriptsize] at (3.580,1.710) {$B_2$};
\draw[->,black!60] (-1.00,-1.00) -- (-0.42,-1.00) node[right,font=\tiny,black!60] {mode 1};
\draw[->,black!60] (-1.00,-1.00) -- (-0.77,-0.835) node[right,font=\tiny,black!60,yshift=2.5pt] {mode 2};
\draw[->,black!60] (-1.00,-1.00) -- (-1.00,-0.44) node[above,font=\tiny,black!60] {mode 3};
\node[circle,draw,inner sep=1.1pt,font=\tiny] (h1) at (0.0,2.85) {1};
\node[circle,draw,inner sep=1.1pt,font=\tiny] (h2) at (0.62,2.85) {2};
\node[font=\tiny,black!60] at (0.31,3.12) {$\Gamma$: no edge};
\node[font=\small] at (-0.95,3.15) {(b)};
\end{tikzpicture}
\caption{Boxes and rooks in the signature space $\Omega$, for two indices
$r\in\{1,2\}$ with facet boxes $B_1=\{1,2\}\times\{1,2\}\times\{1\}$ (blue) and
an orange box $B_2$.
(a) $B_2=\{3\}\times\{2\}\times\{1,2\}$: the mode-$1$ signatures $\{1,2\}$ and
$\{3\}$ are disjoint, but the signatures intersect in modes $2$ and $3$, so the
pair is facet-opposite in exactly one mode, i.e. a bridge edge
(Definition~\ref{def:bridge}).  Correspondingly, one rook step changing only the
first coordinate crosses from $B_1$ to $B_2$, as in the proof of
Lemma~\ref{lem:bridge-rook}: the boxes lie in a single rook component.
(b) $B_2=\{3\}\times\{3\}\times\{1,2\}$: the signatures are disjoint in modes
$1$ \emph{and} $2$, no bridge edge exists, and every step between the boxes would
have to change two coordinates at once; the union splits into two rook components
$K_1\sqcup K_2$.  Since nonnegative entries cannot cancel, the support of the
exchanged tensor is the union of all boxes, each competing box $B_s'$ (grey) lies
inside a single component, and restricting the exchange to $K_1$ and $K_2$
produces the subexchanges of Lemma~\ref{lem:rect-split}: the exchange is reducible.}
\label{fig:rook}
\end{figure*}

\begin{lemma}[Rectangular splitting]\label{lem:rect-split}
Let
\[
\sum_{r\in S}X_r=\sum_{s\in J}Y_s
\]
be a positive exchange.  Choose nonnegative factors
\[
X_r=x_r^{(1)}\otimes\cdots\otimes x_r^{(d)},
\qquad
Y_s=y_s^{(1)}\otimes\cdots\otimes y_s^{(d)}
\]
according to Convention~\ref{conv:nonneg-factors}.  Suppose that for every
mode $j$ there is an admissible subspace
$W_j\subseteq\R^{n_j}$ containing all mode-$j$ factors of both sides.  If
the bridge graph of the family
$(x_r^{(j)})_{r\in S}$ with respect to $(W_j)_j$ is disconnected, then the
exchange is reducible.
\end{lemma}

\begin{proof}
All factors are nonnegative vectors in $W_j$, hence belong to $C(W_j)$.
Consider the tensor-product map
\[
\Lambda
=
L_{W_1}\otimes\cdots\otimes L_{W_d}:
W_1\otimes\cdots\otimes W_d
\longrightarrow
\R^{\mathcal F(W_1)}
\otimes\cdots\otimes
\R^{\mathcal F(W_d)}
\cong
\R^\Omega.
\]
Each $L_{W_j}$ is injective by Lemma~\ref{lem:facet-map}, so $\Lambda$ is
injective.

For a product tensor
\[
z^{(1)}\otimes\cdots\otimes z^{(d)},
\qquad
z^{(j)}\in C(W_j)\setminus\{0\},
\]
we have
\[
\Lambda
\bigl(
z^{(1)}\otimes\cdots\otimes z^{(d)}
\bigr)
=
L_{W_1}z^{(1)}
\otimes\cdots\otimes
L_{W_d}z^{(d)}.
\]
This is entrywise nonnegative on $\Omega$, and its support is the box
\[
\supp(L_{W_1}z^{(1)})
\times\cdots\times
\supp(L_{W_d}z^{(d)}).
\]
In particular,
\[
\supp(\Lambda X_r)=B_r,
\]
and
\[
\supp(\Lambda Y_s)
=
B_s'
:=
\prod_{j=1}^d
\mathcal A_{W_j}\bigl(y_s^{(j)}\bigr),
\]
where all these boxes are nonempty.

Applying $\Lambda$ to the exchange gives
\[
M:=\sum_{r\in S}\Lambda X_r
=
\sum_{s\in J}\Lambda Y_s.
\]
Since both sides are sums of entrywise nonnegative tensors, no cancellation
can occur, and therefore
\[
\supp M
=
\bigcup_{r\in S}B_r
=
\bigcup_{s\in J}B_s'
=:U.
\]
In particular, every competing box $B_s'$ is contained in $U$.

Now decompose
\[
U=K_1\sqcup\cdots\sqcup K_c
\]
into rook-connected components.  Since the bridge graph is disconnected,
Lemma~\ref{lem:bridge-rook} gives $c\ge2$.  Every prescribed box $B_r$ and
every competing box $B_s'$ is rook-connected, so each is contained in a
single component.  Define
\[
S_\ell
=
\{r\in S:B_r\subseteq K_\ell\},
\qquad
J_\ell
=
\{s\in J:B_s'\subseteq K_\ell\},
\qquad
\ell=1,\dots,c.
\]
These sets partition $S$ and $J$.  Each $S_\ell$ is nonempty because every
point of $K_\ell$ belongs to some prescribed box $B_r$, and that entire
box is rook-connected and hence contained in $K_\ell$.

For each $\ell$, put
\[
M_\ell
=
\sum_{r\in S_\ell}\Lambda X_r,
\qquad
N_\ell
=
\sum_{s\in J_\ell}\Lambda Y_s.
\]
Both tensors are supported in $K_\ell$.  At every point of $K_\ell$, all
terms belonging to other components vanish, so
\[
M_\ell=M=N_\ell
\]
on $K_\ell$; both sides vanish outside $K_\ell$.  Hence
$M_\ell=N_\ell$ on all of $\Omega$.

Moreover $J_\ell\neq\varnothing$.  Otherwise $N_\ell=0$, whereas
\[
\supp M_\ell
=
\bigcup_{r\in S_\ell}B_r
\neq\varnothing,
\]
a contradiction.

Since $\Lambda$ is injective,
\[
\sum_{r\in S_\ell}X_r
=
\sum_{s\in J_\ell}Y_s,
\qquad
\ell=1,\dots,c.
\]
Because $c\ge2$, at least one such pair is nonempty and proper.  It is
therefore a nontrivial subexchange, so the original exchange is reducible.
\end{proof}

The significance of Lemma~\ref{lem:rect-split} is that irreducibility imposes
a connectivity requirement on the factor signatures.  This requirement is
independent of linear dimension: it arises solely because nonnegative
product tensors have rectangular supports in the signature space and cannot
cancel outside those supports.  In the next section, we quantify the amount
of additional factor-space dimension needed to satisfy this connectivity
requirement.


\section{The Positive Scattering Term}\label{sec:tau}

The preceding section showed that irreducibility of a positive exchange
forces connectivity of the associated bridge graph.  This suggests measuring
how far the minimal factor spaces are from being connected: how many
additional factor-space dimensions are needed to reconnect the prescribed
components while remaining inside the support hulls allowed by
nonnegativity?  The positive scattering term makes this quantity precise.

\subsection{Support Confinement and the Definition of $\tau$}
\label{subsec:tau-def}

Return to the decomposition \eqref{eq:decomposition} and fix
$S\subseteq[R]$ with $|S|\ge2$.  For each mode define the \emph{support hull}
\[
N_j(S)
=
\bigcup_{r\in S}\supp\bigl(a_r^{(j)}\bigr)
\subseteq[n_j],
\qquad
H_j(S)
=
\spann\{e_i:i\in N_j(S)\},
\]
and write
\[
h_j(S)=\dim H_j(S)=|N_j(S)|.
\]
Then
\[
U_j(S)\subseteq H_j(S),
\]
and $H_j(S)$ is admissible because it is spanned by standard basis vectors.

The first observation is that the support hull is not merely a convenient
restriction: it is forced by any nonnegative exchange involving the
components indexed by $S$.

\begin{lemma}[Support confinement]\label{lem:support-confinement}
Let
\[
\sum_{r\in S}c_rP_r=\sum_{s\in J}Q_s
\]
be a positive exchange with $c_r>0$, and write
\[
Q_s=b_s^{(1)}\otimes\cdots\otimes b_s^{(d)}
\]
with nonnegative factors.  Then
\[
\supp\bigl(b_s^{(j)}\bigr)\subseteq N_j(S),
\qquad
s\in J,\quad j\in[d].
\]
Equivalently,
\[
b_s^{(j)}\in H_j(S)
\qquad\text{for all }s\in J,\ j\in[d].
\]
\end{lemma}

\begin{proof}
Suppose, to the contrary, that for some $s\in J$ and some mode $j$ there is
\[
i_0\in\supp\bigl(b_s^{(j)}\bigr)\setminus N_j(S).
\]
For every mode $l\neq j$, choose
\[
i_l\in\supp\bigl(b_s^{(l)}\bigr),
\]
which is possible because the factors are nonzero.  Consider the entry of
the exchange indexed by $(i_1,\ldots,i_d)$.  The term $Q_s$ contributes
\[
b_s^{(j)}(i_0)\prod_{l\neq j}b_s^{(l)}(i_l)>0,
\]
so the right-hand side is strictly positive.  On the left, every term
$c_rP_r$ vanishes at this index because
$i_0\notin N_j(S)$ implies
$a_r^{(j)}(i_0)=0$ for every $r\in S$.  This is a contradiction.
\end{proof}

Thus every competing nonnegative factor in an exchange involving $S$ is
confined to the coordinate subspaces $H_j(S)$.  We therefore measure
connectivity only through intermediate spaces lying between the prescribed
factor spans and these support hulls.

Call a tuple of subspaces
\[
(W_1,\ldots,W_d)
\]
\emph{admissible for $S$} if, for every $j$,
\[
U_j(S)\subseteq W_j\subseteq H_j(S)
\]
and $W_j$ is admissible.  Call the tuple \emph{feasible for $S$} if its
associated bridge graph, as defined in Section~\ref{sec:bridge}, is
connected.

\begin{definition}[Positive scattering]
\label{def:tau}
The \emph{positive scattering term} of $S$ is
\begin{equation}
\tau(S)
=
\min\left\{
\sum_{j=1}^d
\bigl(\dim W_j-d_j(S)\bigr):
(W_1,\ldots,W_d)
\text{ is feasible for }S
\right\},
\label{eq:tau-definition}
\end{equation}
with $\tau(S)=+\infty$ if no feasible tuple exists.
\end{definition}

The interpretation is straightforward.  The minimal choice
\[
W_j=U_j(S)
\]
uses no additional dimensions and gives the original bridge graph
$\Gamma_0(S)$.  Enlarging $W_j$ can create new facet intersections and
thereby add bridge edges.  The quantity $\tau(S)$ is the minimum total
number of dimensions that must be added across the modes before the bridge
graph becomes connected.

\begin{remark}[Well-definedness and invariance]
\label{rem:tau-well-defined}
Whenever a feasible tuple exists, the minimum in
\eqref{eq:tau-definition} is attained because the possible costs are
nonnegative integers bounded above by
\[
\sum_{j=1}^d\bigl(h_j(S)-d_j(S)\bigr).
\]
The value $\tau(S)$ is invariant under positive rescaling of the factors:
such rescaling changes neither the spaces $U_j(S)$ and $H_j(S)$ nor the
facet signatures.  It is also invariant under relabeling of the components
and under permutations of coordinates within a mode.  Finally, ambient
coordinates that are identically zero do not affect $\tau(S)$, because
support confinement restricts attention to the support hulls.
\end{remark}

The remainder of this section gives an exact finite characterization of
\eqref{eq:tau-definition}.  The key observation is that connectivity can be
built edge by edge, so the continuous optimization over subspaces can first
be separated by mode and then reduced to spanning trees.

\subsection{Mode Costs and the Tree Formula}\label{subsec:tau-tree}

Fix a mode $j$ and a finite set $F$ of unordered pairs of elements of $S$.
Define
\begin{multline*}
\kappa_j(F)
=
\min\Bigl\{
\dim W-d_j(S):
W\text{ admissible},\
U_j(S)\subseteq W\subseteq H_j(S),\\
\mathcal A_W\bigl(a_r^{(j)}\bigr)
\cap
\mathcal A_W\bigl(a_s^{(j)}\bigr)
\neq\varnothing
\quad
\text{for every }\{r,s\}\in F
\Bigr\},
\end{multline*}
with $\kappa_j(F)=+\infty$ if no such $W$ exists, and
$\kappa_j(\varnothing)=0$.

The quantity $\kappa_j(F)$ is the minimum number of dimensions that must be
added in mode $j$ in order to make all pairs in $F$ have intersecting facet
signatures.  The following lemma gives the first simplification.

\begin{lemma}[Finiteness of $\kappa_j(F)$]\label{lem:kappa-finite}
For every finite set $F$ of pairs,
$\kappa_j(F)<\infty$ if and only if
\[
\supp\bigl(a_r^{(j)}\bigr)
\cap
\supp\bigl(a_s^{(j)}\bigr)
\neq\varnothing
\]
for every $\{r,s\}\in F$.  When finite,
\[
0\le\kappa_j(F)\le h_j(S)-d_j(S).
\]
\end{lemma}

\begin{proof}
Suppose first that $\kappa_j(F)<\infty$, and let $W$ be admissible for
which all required signature intersections are nonempty.  For
$\{r,s\}\in F$, choose
\[
F_0\in
\mathcal A_W\bigl(a_r^{(j)}\bigr)
\cap
\mathcal A_W\bigl(a_s^{(j)}\bigr).
\]
The corresponding coordinate is strictly positive for both factors, so
their ordinary supports intersect.

Conversely, suppose all required pairs have intersecting ordinary
supports.  Take
\[
W=H_j(S).
\]
Then $C(W)$ is the full nonnegative orthant on the support coordinates
$N_j(S)$, and its facets are exactly the coordinate hyperplanes
$x_i=0$, $i\in N_j(S)$.  Hence facet signatures coincide with ordinary
supports.  All required signature intersections therefore hold, and the
cost is
\[
h_j(S)-d_j(S).
\]
The lower bound is immediate from $W\supseteq U_j(S)$.
\end{proof}

It is useful to view a pair as being resolved in a mode when its two
signatures intersect.  The next result shows that only a spanning tree of
such pairwise requirements is needed.

For a spanning tree $T$ of the complete graph on $S$ and a map
\[
\varepsilon:E(T)\to[d],
\]
call $\varepsilon(e)$ the \emph{exempted mode} of edge $e$.  Define
\[
F_j(T,\varepsilon)
=
\{e\in E(T):\varepsilon(e)\neq j\}.
\]
Thus mode $j$ is required to resolve every tree edge except those exempted
to $j$.

\begin{proposition}[Tree formula]\label{prop:tau-tree}
For every $S\subseteq[R]$ with $|S|\ge2$,
\begin{equation}
\tau(S)
=
\min_{(T,\varepsilon)}
\sum_{j=1}^d
\kappa_j\bigl(F_j(T,\varepsilon)\bigr),
\label{eq:tau-tree}
\end{equation}
where $(T,\varepsilon)$ ranges over all spanning trees of the complete graph
on $S$ and all exemption maps.  Both sides may equal $+\infty$.
\end{proposition}

\begin{proof}
Let $\rho$ denote the right-hand side.

\emph{$\rho\le\tau(S)$.}
Assume $\tau(S)<\infty$ and let
$(W_j)_j$ be a feasible tuple attaining the minimum.  Its bridge graph is
connected, so it contains a spanning tree $T$.  For every edge
$e=\{r,s\}\in E(T)$, choose an exempted mode $\varepsilon(e)$ in which
$r,s$ are allowed to be facet-opposite.  Such a mode exists because
$e$ is a bridge edge.  Hence, for every $j\neq\varepsilon(e)$,
\[
\mathcal A_{W_j}\bigl(a_r^{(j)}\bigr)
\cap
\mathcal A_{W_j}\bigl(a_s^{(j)}\bigr)
\neq\varnothing.
\]
Thus $W_j$ is admissible in the definition of
$\kappa_j(F_j(T,\varepsilon))$, and
\[
\kappa_j(F_j(T,\varepsilon))
\le
\dim W_j-d_j(S).
\]
Summing over $j$ gives
\[
\rho\le\tau(S).
\]

\emph{$\tau(S)\le\rho$.}
Suppose $\rho<\infty$ and fix $(T,\varepsilon)$ attaining the minimum.
For each mode choose an admissible $W_j$ attaining
\[
\kappa_j(F_j(T,\varepsilon)).
\]
For every edge $e=\{r,s\}\in E(T)$ and every
$j\neq\varepsilon(e)$, the signatures of
$a_r^{(j)}$ and $a_s^{(j)}$ intersect.  Thus $e$ is facet-opposite in at
most the one mode $\varepsilon(e)$, so every edge of $T$ is a bridge edge
for the tuple $(W_j)_j$.  The resulting bridge graph contains the spanning
tree $T$ and is therefore connected.  Hence the tuple is feasible and
\[
\tau(S)
\le
\sum_{j=1}^d
\bigl(\dim W_j-d_j(S)\bigr)
=
\rho.
\]
\end{proof}

For two components, the formula takes a particularly simple form.  If
$S=\{r,s\}$, the unique spanning tree consists of the single edge
$e=\{r,s\}$.  Exempting mode $j_0$ gives
\[
F_j=
\begin{cases}
\{e\},&j\neq j_0,\\
\varnothing,&j=j_0,
\end{cases}
\]
and therefore
\begin{equation}
\tau(\{r,s\})
=
\min_{j_0\in[d]}
\sum_{j\neq j_0}
\kappa_j\bigl(\{e\}\bigr).
\label{eq:tau-pair}
\end{equation}

\subsection{Exact Mode Costs}\label{subsec:exact-kappa}

The remaining question is how difficult the mode costs
$\kappa_j(F)$ are to compute.  The answer is particularly simple:
for each pair, the cost is always $0$, $1$, or $+\infty$.  We establish
this in one mode and suppress the mode index and the set $S$.

Delete coordinates that vanish identically on the support hull and write
\[
H=\R^m,
\qquad
U=\spann\{a_r:r\in S\}\subseteq H,
\qquad
p=\sum_{r\in S}a_r.
\]
By construction,
\[
p_i>0,\qquad i\in[m].
\]
For an intermediate space $U\subseteq W\subseteq H$, define
\[
E(W)
=
\{i\in[m]:
C(W)\cap\{x_i=0\}
\text{ is a facet of }C(W)\}.
\]
If several coordinates define the same facet, all such coordinates belong
to $E(W)$.

\begin{lemma}[Automatic admissibility]\label{lem:auto-admissible}
Every intermediate subspace
\[
U\subseteq W\subseteq H
\]
is admissible.
\end{lemma}

\begin{proof}
Let $w\in W$.  Since every coordinate of $p$ is strictly positive, there is
$t>0$ such that
\[
tp_i+w_i\ge0,
\qquad i\in[m].
\]
For example, one may take
\[
t\ge
\max_{i:\,w_i<0}\frac{-w_i}{p_i},
\]
with the maximum over the empty set interpreted as zero.  Then both
$tp$ and $tp+w$ belong to $W\cap\Rnn^m$, and
\[
w=(tp+w)-tp.
\]
Thus every vector of $W$ lies in the linear span of its nonnegative part.
\end{proof}

For each coordinate, define the normalized functional
\[
\ell_i^W=e_i^*|_W,
\qquad
q_i^W=\frac{\ell_i^W}{p_i}.
\]
Since $q_i^W(p)=1$, all such functionals lie in the affine hyperplane
\[
\{f\in W^*:f(p)=1\}.
\]

\begin{proposition}[Normalized dual polytope]
\label{prop:normalized-dual-polytope}
Let $U\subseteq W\subseteq H$.
\begin{enumerate}[(i)]
\item
\[
C(W)^*
=
\cone\{\ell_i^W:i\in[m]\}.
\]
\item The section
\[
P(W)=\{f\in C(W)^*:f(p)=1\}
\]
satisfies
\[
P(W)=\conv\{q_i^W:i\in[m]\}.
\]
\item A coordinate $i$ belongs to $E(W)$ if and only if $q_i^W$ is a vertex
of $P(W)$.
\item For nonzero $x,y\in C(W)$,
\[
\mathcal A_W(x)\cap\mathcal A_W(y)\neq\varnothing
\quad\Longleftrightarrow\quad
E(W)\cap\supp(x)\cap\supp(y)\neq\varnothing.
\]
\end{enumerate}
\end{proposition}

\begin{proof}
\emph{(i).}
The inclusion
\[
\cone\{\ell_i^W:i\in[m]\}\subseteq C(W)^*
\]
is immediate because each coordinate functional is nonnegative on $C(W)$.
For the reverse inclusion, let
\[
K=\cone\{\ell_i^W:i\in[m]\}
\]
and suppose $f\in C(W)^*\setminus K$.  Since $K$ is a closed polyhedral
cone, there exists $x\in W$ such that
\[
g(x)\ge0\quad\text{for all }g\in K,
\qquad
f(x)<0.
\]
In particular,
\[
x_i=\ell_i^W(x)\ge0
\qquad\text{for all }i,
\]
so $x\in C(W)$.  This contradicts $f\in C(W)^*$.  Hence
$C(W)^*=K$.

\emph{(ii).}
Write
\[
f=\sum_i\alpha_i\ell_i^W,
\qquad
\alpha_i\ge0.
\]
Since $f(p)=1$,
\[
1=\sum_i\alpha_i p_i.
\]
Putting
\[
\beta_i=\alpha_i p_i
\]
gives $\beta_i\ge0$, $\sum_i\beta_i=1$, and
\[
f=\sum_i\beta_iq_i^W.
\]
Thus
\[
P(W)=\conv\{q_i^W:i\in[m]\}.
\]

\emph{(iii).}
Because every coordinate of $p$ is strictly positive, $p$ lies in
$\relint C(W)$.  Hence every nonzero element of $C(W)^*$ is strictly
positive at $p$, and the section $P(W)$ meets each nonzero ray of
$C(W)^*$ exactly once.  Under this correspondence, extreme rays of
$C(W)^*$ are precisely the vertices of $P(W)$.  Since facets of the
full-dimensional pointed cone $C(W)$ correspond to extreme rays of its
dual cone, and the coordinate ray generated by $\ell_i^W$ defines a facet
exactly when $i\in E(W)$, the equivalence follows.

\emph{(iv).}
A facet belongs to both signatures precisely when some coordinate
representative $i\in E(W)$ is strictly positive for both $x$ and $y$.
If multiple coordinates define the same facet, their restrictions to $W$
are positive multiples of one another by Lemma~\ref{lem:intrinsic-cone}(iii),
so the condition is independent of the representative.
\end{proof}

The key geometric fact is that, once $U$ is enlarged at all, one additional
dimension is enough to make every effective coordinate define a facet.
Geometrically, the normalized coordinate functionals form a finite
configuration in an affine hyperplane.  One additional height coordinate can
be chosen so that all relevant points become exposed vertices of the lifted
convex hull.  This is what turns the apparently continuous enlargement
problem into the discrete $0/1/+\infty$ trichotomy below.

\begin{proposition}[One-dimensional facet saturation]
\label{prop:saturation}
If
\[
U\subsetneq H,
\]
there exists $z\in H\setminus U$ such that
\[
W=U+\R z
\]
satisfies
\[
\dim W=\dim U+1
\qquad\text{and}\qquad
E(W)=[m].
\]
\end{proposition}

\begin{proof}
Put
\[
k=\dim U
\]
and consider the affine hyperplane
\[
A=\{b\in U^*:b(p)=1\}.
\]
For each coordinate define
\[
b_i=\frac{e_i^*|_U}{p_i}\in A.
\]
These points affinely span $A$. Let $\operatorname{Aff}(A)$ denote the vector space of real affine functions
on $A$. Indeed, the map
\[
\Phi:U\longrightarrow\operatorname{Aff}(A),
\qquad
\Phi(u)(b)=b(u),
\]
is injective: if $\Phi(u)=0$, then
\[
0=b_i(u)=\frac{u_i}{p_i}
\]
for every $i$, hence $u=0$.  Since both spaces have dimension $k$, $\Phi$
is an isomorphism.  Thus a nonzero affine function cannot vanish at every
$b_i$, which proves that the $b_i$ affinely span $A$.

Let
\[
b^{(1)},\ldots,b^{(g)}
\]
be the distinct points among the $b_i$.  Since they affinely span the
$(k-1)$-dimensional space $A$, we have $g\ge k$.

Assign heights $h_i$ and set
\[
z_i=p_i h_i.
\]
If $z\notin U$, then
\[
W=U\oplus\R z.
\]
Under the identification
\[
\{f\in W^*:f(p)=1\}\cong A\times\R,
\]
the normalized coordinate functional $q_i^W$ corresponds to the lifted
point
\[
(b_i,h_i).
\]

If $g>k$, choose an inner product on $A$ and set
\[
h_\alpha^0=\|b^{(\alpha)}\|^2.
\]
For each $\alpha$, the affine function
\[
L_\alpha(x)
=
2\langle b^{(\alpha)},x\rangle
-
\|b^{(\alpha)}\|^2
\]
satisfies
\[
\|b^{(\beta)}\|^2-L_\alpha(b^{(\beta)})
=
\|b^{(\beta)}-b^{(\alpha)}\|^2>0,
\qquad
\beta\neq\alpha.
\]
Thus every lifted point $(b^{(\alpha)},h_\alpha^0)$ is strictly exposed
from below.  If these heights are already non-affine in the base points,
we are done.  Otherwise, choose a class-height perturbation
$\delta=(\delta_\alpha)$ outside the space of affine height vectors and put
\[
h_\alpha=h_\alpha^0+\varepsilon\delta_\alpha
\]
for sufficiently small $\varepsilon>0$.  The strict exposure inequalities
persist, while the perturbed heights are no longer affine.

If $g=k$, then the distinct base points are affinely independent.  Since
$U\subsetneq H$, we have $m>k$, so some base point occurs for at least two
coordinates; say
\[
b_{i_0}=b_{i_1}.
\]
Set
\[
h_{i_0}=1,
\qquad
h_i=0\quad(i\neq i_0).
\]
The distinct lifted points are then
\[
(b^{(1)},1),\quad
(b^{(1)},0),\quad
(b^{(2)},0),\ldots,(b^{(k)},0),
\]
after relabeling.  Their difference vectors are linearly independent, so
they are the vertices of a $k$-simplex.  Again the heights are not affine.

In either case, $z\notin U$ and every normalized coordinate functional is
a vertex of $P(W)$.  Proposition~\ref{prop:normalized-dual-polytope}
therefore gives
\[
E(W)=[m].
\]
\end{proof}

\begin{proposition}[Exact mode-cost trichotomy]
\label{prop:kappa-trichotomy}
For every finite set $F$ of unordered pairs,
\[
\kappa(F)=
\begin{cases}
+\infty,
&
\text{if some }\{r,s\}\in F
\text{ has disjoint ordinary supports},\\[4pt]
0,
&
\text{if }
\mathcal A_U(a_r)\cap\mathcal A_U(a_s)\neq\varnothing
\text{ for every }\{r,s\}\in F,\\[4pt]
1,
&
\text{otherwise}.
\end{cases}
\]
Here $\kappa(\varnothing)=0$.
\end{proposition}

\begin{proof}
If some required pair has disjoint ordinary supports, Lemma~\ref{lem:kappa-finite}
gives $\kappa(F)=+\infty$.

If all required signatures already intersect in $U$, then $W=U$ is feasible
with cost zero, so $\kappa(F)=0$.

It remains to consider the case in which every required pair has nonempty
ordinary support intersection but at least one pair has disjoint signatures
in $U$.  Then cost zero is impossible because a zero-cost space must equal
$U$.  Moreover $U\neq H$: if $U=H$, the signatures in $U$ are simply ordinary
supports, contradicting the assumption that some required pair has
disjoint signatures.  Hence Proposition~\ref{prop:saturation} provides a
one-dimensional extension with $E(W)=[m]$.  Every pair with intersecting
ordinary supports then has intersecting signatures by
Proposition~\ref{prop:normalized-dual-polytope}, so $\kappa(F)=1$.
\end{proof}

Since the only finite values are $0$ and $1$, the cost for a collection of
pairs is determined by the worst pair:
\begin{equation}
\kappa_j(F)
=
\max_{e\in F}\kappa_j(e),
\qquad
F\neq\varnothing,
\label{eq:kappa-max}
\end{equation}
where
\[
\kappa_j(e)
=
\kappa_j(\{e\})
\in\{0,1,+\infty\}.
\]

\subsection{The Activation Formula}\label{subsec:activation}

The tree formula and the trichotomy together turn the definition of
$\tau(S)$ into a finite graph problem.  For
$A\subseteq[d]$, interpret the modes in $A$ as \emph{activated}: in an
activated mode we allow the one-dimensional extension supplied by
Proposition~\ref{prop:saturation}.

Define $G_A(S)$ to be the graph on $S$ in which
$e=\{r,s\}$ is an edge whenever
\[
\#\Bigl\{
j\in[d]:
\kappa_j(e)=+\infty
\ \text{or}\
\bigl(\kappa_j(e)=1\text{ and }j\notin A\bigr)
\Bigr\}
\le1.
\]
Thus a pair is a bridge precisely when, after activation of the modes in
$A$, there is at most one mode in which its connectivity remains
unresolved.

\begin{theorem}[Activation formula]\label{thm:tau-activation}
Let $S\subseteq[R]$ with $|S|\ge2$.
\begin{enumerate}[(i)]
\item
\[
\tau(S)
=
\min\bigl\{
|A|:A\subseteq[d],\ G_A(S)\text{ is connected}
\bigr\},
\]
with the minimum over an empty family interpreted as $+\infty$.
\item
\[
G_\varnothing(S)=\Gamma_0(S),
\]
and hence
\[
\tau(S)=0
\quad\Longleftrightarrow\quad
\Gamma_0(S)\text{ is connected}.
\]
\item $\tau(S)<\infty$ if and only if $G_{[d]}(S)$ is connected.  In
particular,
\[
\tau(S)
\le
\#\{j:U_j(S)\neq H_j(S)\}
\le d
\]
whenever $\tau(S)<\infty$.
\end{enumerate}
\end{theorem}

\begin{proof}
\emph{Part (i).}
Suppose $G_A(S)$ is connected.  For every inactive mode $j\notin A$, take
\[
W_j=U_j(S).
\]
For every active mode $j\in A$ with $U_j(S)\subsetneq H_j(S)$, choose a
one-dimensional saturated extension from
Proposition~\ref{prop:saturation}; if $U_j(S)=H_j(S)$, retain
$W_j=U_j(S)$.  The resulting tuple has cost at most $|A|$.

In an inactive mode, a pair with $\kappa_j(e)=0$ has intersecting
signatures, while a pair with $\kappa_j(e)\ge1$ is facet-opposite.  In an
active mode, saturation resolves every pair having nonempty ordinary
support intersection, leaving only pairs with
$\kappa_j(e)=+\infty$.  Therefore the bridge graph of the constructed tuple
is exactly $G_A(S)$ and is connected.  Hence
\[
\tau(S)\le |A|.
\]
Minimizing over connected $G_A(S)$ gives
\[
\tau(S)
\le
\min\{|A|:G_A(S)\text{ connected}\}.
\]

Conversely, let $(W_j)_j$ be a feasible tuple and define
\[
A=\{j:W_j\neq U_j(S)\}.
\]
Every active mode contributes at least one dimension, so
\[
|A|
\le
\sum_j\bigl(\dim W_j-d_j(S)\bigr).
\]
Replace each active space by a one-dimensional saturated extension of
$U_j(S)$.  This cannot destroy any bridge edge: if a pair had intersecting
signatures in the original space, its ordinary supports intersect, and
saturation preserves that intersection; if its ordinary supports are
disjoint, no admissible extension can resolve the pair.  Hence the resulting
bridge graph is $G_A(S)$ and remains connected.  Therefore
\[
\min\{|A|:G_A(S)\text{ connected}\}
\le
\sum_j\bigl(\dim W_j-d_j(S)\bigr).
\]
Minimizing over feasible tuples proves (i).

\emph{Part (ii).}
By the trichotomy,
\[
\kappa_j(e)=0
\]
holds exactly when the signatures of $a_r^{(j)}$ and $a_s^{(j)}$ intersect
in the minimal space $U_j(S)$.  Consequently an edge of
$G_\varnothing(S)$ is precisely a pair that is facet-opposite in at most
one mode, which is the definition of an edge of $\Gamma_0(S)$.

\emph{Part (iii).}
When $A=[d]$, the cost-one obstruction is always activated, so an edge
fails to occur only if
\[
\kappa_j(e)=+\infty
\]
in at least two modes.  By Lemma~\ref{lem:kappa-finite}, this is exactly the
case in which the ordinary supports are disjoint in at least two modes.
Thus $G_{[d]}(S)$ is the graph obtained by requiring ordinary support
intersections in all but at most one mode.

If $\tau(S)<\infty$, then some $G_A(S)$ is connected, hence so is
$G_{[d]}(S)$ because the graphs are monotone in $A$.  Conversely, suppose
$G_{[d]}(S)$ is connected and put
\[
A^*=\{j:U_j(S)\neq H_j(S)\}.
\]
If $U_j(S)=H_j(S)$, then the signatures in mode $j$ equal ordinary supports,
so no pair has $\kappa_j(e)=1$ in that mode.  Therefore
\[
G_{A^*}(S)=G_{[d]}(S),
\]
which is connected.  Part (i) gives
\[
\tau(S)\le |A^*|
=
\#\{j:U_j(S)\neq H_j(S)\}
\le d.
\]
\end{proof}

\begin{remark}[Exact computation]
\label{rem:effective}
Theorem~\ref{thm:tau-activation} gives an exact finite procedure for
computing $\tau(S)$.  First, for each mode, ordinary support intersections
determine which pair costs are $+\infty$.  For a support-intersecting pair,
the distinction between costs $0$ and $1$ is determined by whether the
corresponding normalized coordinate functionals are vertices of
$P(U_j(S))$, by Proposition~\ref{prop:normalized-dual-polytope}.  After this
preprocessing, $\tau(S)$ is obtained by checking activation sets
$A\subseteq[d]$ in increasing order of cardinality until
$G_A(S)$ becomes connected.

For rational input data, the vertex tests can be formulated as linear
programming feasibility problems.  Thus, for fixed tensor order $d$, the
postprocessing after the vertex tests consists of at most $2^d$ graph
connectivity problems and is polynomial in $|S|$.  The full identifiability
certificate still quantifies over all subsets $S\subseteq[R]$, so the
criterion as a whole need not be polynomial-time in $R$.
\end{remark}

The resulting computation is summarized in
Algorithm~\ref{alg:compute-tau}.  The algorithm first determines the
facet-coordinate sets for the minimal spaces $U_j(S)$, then computes the
pair costs $\kappa_j(e)$, and finally searches over activation sets.

\begin{algorithm}[htbp]
\caption{Exact computation of $\tau(S)$}
\label{alg:compute-tau}
\begin{algorithmic}[1]
\Require $S\subseteq[R]$, $|S|\ge2$, and the factors $a_r^{(j)}$.
\Ensure $\tau(S)$.

\For{$j\in[d]$}
\State $p^{(j)}\gets\sum_{r\in S}a_r^{(j)}$ and
$q_i^{(j)}\gets(e_i^*|_{U_j(S)})/p_i^{(j)}$ for $i\in N_j(S)$.
\State Group equal $q_i^{(j)}$'s and determine, by LP vertex tests,
\Statex \hspace{\algorithmicindent}
$E_j\gets
\{i\in N_j(S):
q_i^{(j)}
\in\operatorname{vert}\conv\{q_\ell^{(j)}:\ell\in N_j(S)\}\}$.
\ForAll{$e=\{r,s\}\in\binom{S}{2}$}
\State $I_j(e)\gets
\supp(a_r^{(j)})\cap\supp(a_s^{(j)})$.
\State
$\displaystyle
\kappa_j(e)\gets
\begin{cases}
	+\infty,&I_j(e)=\varnothing,\\
	0,&I_j(e)\cap E_j\ne\varnothing,\\
	1,&\text{otherwise}.
\end{cases}$
\EndFor
\EndFor

\ForAll{$A\subseteq[d]$, in nondecreasing order of $|A|$}
\State $G_A(S)\gets
\bigl(S,\{e\in\binom{S}{2}:b_A(e)\le1\}\bigr)$, where
\Statex \hspace{\algorithmicindent}
$b_A(e)=
\#\{j:\kappa_j(e)=+\infty
\text{ or }(\kappa_j(e)=1\text{ and }j\notin A)\}$.
\If{$G_A(S)$ is connected}
\State \Return $|A|$.
\EndIf
\EndFor
\State \Return $+\infty$.
\end{algorithmic}
\end{algorithm}



\section{Proof of the Main Criterion}\label{sec:proofs}

The proof of Theorem~\ref{thm:main} combines the two mechanisms developed
above.  The Lovitz--Petrov splitting theorem gives a constraint on the
linear dimension of an irreducible exchange, while the support geometry
forces the factor spaces of that exchange to pay the additional scattering
cost $\tau(S)$.  The two effects combine in the positive splitting
inequality below.  The minimality and uniqueness statements then follow by
applying this inequality to an irreducible block of a competing
decomposition.

\subsection{The Positive Splitting Inequality}\label{subsec:splitting}

Fix $S\subseteq[R]$ with $|S|\ge2$ and consider an irreducible positive
exchange
\begin{equation}
\sum_{r\in S}c_rP_r
=
\sum_{s\in J}Q_s,
\qquad
c_r>0,
\label{eq:exchange}
\end{equation}
where the left-hand side consists of positive multiples of the prescribed
terms indexed by $S$ and the right-hand side consists of nonzero
nonnegative rank-one tensors
\[
Q_s=b_s^{(1)}\otimes\cdots\otimes b_s^{(d)}
\]
with nonnegative factors.

Put
\[
U=\spann\{P_r:r\in S\},
\qquad
V=\spann\{Q_s:s\in J\},
\qquad
h=\dim(U\cap V),
\]
and, for each mode,
\[
V_j=\spann\{b_s^{(j)}:s\in J\},
\qquad
W_j=U_j(S)+V_j.
\]
By the support-confinement lemma, every $b_s^{(j)}$ belongs to
$H_j(S)$.  Hence
\[
U_j(S)\subseteq W_j\subseteq H_j(S).
\]
Moreover, each $W_j$ is admissible because it is spanned by nonnegative
vectors.  Thus $(W_j)_j$ is an admissible tuple for $S$.

Finally, $h\ge1$: the common value of the two sides of
\eqref{eq:exchange} is a nonzero nonnegative tensor and therefore belongs
to both $U$ and $V$.

\begin{theorem}[Positive splitting inequality]
\label{thm:positive-splitting}
For every irreducible positive exchange \eqref{eq:exchange},
\begin{equation}
\beta(S)+\tau(S)
\le
\dim U+\dim V-2.
\label{eq:positive-splitting}
\end{equation}
In particular, if the two sides contain
$p=|S|$ and $q=|J|$ terms, respectively, then
\[
\beta(S)+\tau(S)\le p+q-2.
\]
\end{theorem}

\begin{proof}[Proof of Theorem~\ref{thm:positive-splitting}]

\emph{Step 1: Irreducibility implies connectedness.}

Absorb a minus sign into one factor of each $Q_s$, say the mode-$1$ factor,
and consider the signed multiset of nonzero product tensors
\[
E
=
\{c_rP_r:r\in S\}
\cup
\{-Q_s:s\in J\}.
\]
The sum of all elements of $E$ is zero.

We claim that $E$ is connected in the sense of
Definition~\ref{def:connected}.  Suppose instead that
\[
E=E_1\sqcup E_2,
\qquad
\spann E
=
\spann E_1\oplus\spann E_2,
\]
with both $E_1$ and $E_2$ nonempty.  Let $\sigma_k$ denote the sum of the
elements of $E_k$.  Then
\[
\sigma_1+\sigma_2=0,
\qquad
\sigma_k\in\spann E_k.
\]
Since the sum is direct,
\[
\sigma_1=\sigma_2=0.
\]
Let $I_k\subseteq S$ and $J_k\subseteq J$ be the indices of the terms of
$E_k$ originating from the two sides of the exchange.  Then
\[
\sum_{r\in I_k}c_rP_r
=
\sum_{s\in J_k}Q_s,
\qquad k=1,2.
\]
Neither $E_k$ can contain terms from only one side: if, for instance,
$J_1=\varnothing$, then
\[
\sum_{r\in I_1}c_rP_r=0,
\]
which is impossible because all terms are nonzero and nonnegative.
Thus
\[
I_1,J_1,I_2,J_2\neq\varnothing.
\]
Consequently $(I_1,J_1)$ is a nontrivial subexchange of
\eqref{eq:exchange}, contradicting irreducibility.  Hence $E$ is connected.

\medskip
\emph{Step 2: The exchange pays the scattering cost.}

The mode-$j$ factors of the left side of \eqref{eq:exchange} may be taken
to be $c_ra_r^{(1)}$ in mode $1$ and $a_r^{(j)}$ in the remaining modes.
Together with the factors $b_s^{(j)}$ on the right, they all belong to
$W_j$.  Positive rescaling does not change facet signatures.  If the
bridge graph of the prescribed factors with respect to $(W_j)_j$ were
disconnected, the rectangular splitting lemma
(Lemma~\ref{lem:rect-split}) would imply that the exchange is reducible.
Hence this bridge graph is connected, so $(W_j)_j$ is feasible for $S$.
By the definition of $\tau(S)$,
\[
\tau(S)
\le
\sum_{j=1}^d
\bigl(\dim W_j-d_j(S)\bigr).
\tag{6.1}\label{eq:tau-upper-exchange}
\]

\medskip
\emph{Step 3: Dimension counting.}

The span of the signed family $E$ is
\[
\spann E=U+V.
\]
In each mode, the factors appearing in $E$ span exactly $W_j$, so the
mode-$j$ rank of $E$ is
\[
r_j=\dim W_j.
\]

Since $E$ is connected, it does not split.  The contrapositive of the
Lovitz--Petrov splitting theorem therefore gives
\[
\dim\spann E
>
\sum_{j=1}^d(r_j-1)
=
\sum_{j=1}^d(\dim W_j-1).
\]
All quantities are integers, hence
\[
\sum_{j=1}^d(\dim W_j-1)
\le
\dim(U+V)-1.
\]
Using
\[
\dim(U+V)
=
\dim U+\dim V-h,
\]
we obtain
\[
\sum_{j=1}^d(\dim W_j-1)
\le
\dim U+\dim V-h-1.
\]
Since
\[
\beta(S)
=
\sum_{j=1}^d(d_j(S)-1),
\]
this can be rewritten as
\[
\beta(S)
+
\sum_{j=1}^d
\bigl(\dim W_j-d_j(S)\bigr)
\le
\dim U+\dim V-h-1.
\]
Combining this with \eqref{eq:tau-upper-exchange} and $h\ge1$ gives
\[
\beta(S)+\tau(S)
\le
\dim U+\dim V-h-1
\le
\dim U+\dim V-2,
\]
which proves \eqref{eq:positive-splitting}.  Since
\[
\dim U\le |S|=p,
\qquad
\dim V\le |J|=q,
\]
the final bound follows.
\end{proof}

\begin{remark}[Mechanism]
\label{rem:mechanism}
The positive splitting inequality combines two distinct obstructions to an
irreducible exchange.  The Lovitz--Petrov argument bounds the linear
dimension of a connected family of product tensors and produces the
dimension budget $\beta(S)$.  Nonnegativity provides a second obstruction:
support confinement restricts the competing factors to the support hulls,
and irreducibility forces the corresponding bridge graph to be connected.
The minimum dimension increment needed to achieve this connectivity is
exactly $\tau(S)$.  Finally, positivity implies that the two sides of the
exchange have a common nonzero tensor, so $h=\dim(U\cap V)\ge1$, producing
the additional unit in the final bound.
\end{remark}

\subsection{Minimality and Uniqueness}\label{subsec:main-thm}

\begin{proof}[Proof of Theorem~\ref{thm:main}]

\emph{Part (i): Minimality.}

Suppose, to the contrary, that
\[
\rankp(\T)<R.
\]
Let
\[
\T=\sum_{s=1}^{q}Q_s
\]
be a minimal nonnegative decomposition, where $q<R$.  Then
\[
\sum_{r=1}^{R}P_r
=
\sum_{s=1}^{q}Q_s
\]
is a positive exchange.  Decompose it into irreducible blocks
\[
(I_k,J_k),
\qquad
k=1,\dots,m,
\]
using Lemma~\ref{lem:exchange-decomposition}(i).

Every $J_k$ is nonempty by Remark~\ref{rem:no-one-sided}.  Since
\[
\sum_{k=1}^m|I_k|
=
R
>
q
=
\sum_{k=1}^m|J_k|,
\]
there exists a block for which
\[
p_k:=|I_k|>|J_k|=:q_k.
\]
In particular,
\[
p_k\ge2.
\]

Applying Theorem~\ref{thm:positive-splitting} to this irreducible block
gives
\[
\beta(I_k)+\tau(I_k)
\le
p_k+q_k-2
\le
2p_k-3.
\]
On the other hand, condition~\eqref{eq:M} applied to $S=I_k$ gives
\[
\beta(I_k)+\tau(I_k)
\ge
2p_k-2,
\]
a contradiction.  Hence
\[
\rankp(\T)=R.
\]
The prescribed decomposition is therefore minimal, and its terms are
linearly independent by Lemma~\ref{lem:minimal-independent}.

\medskip
\emph{Part (ii): Uniqueness.}

Assume condition~\eqref{eq:U}.  Since
\[
2|S|-1\ge2|S|-2,
\]
condition~\eqref{eq:M} also holds, so Part (i) implies
\[
\rankp(\T)=R.
\]

Let
\[
\T=\sum_{s=1}^{R}Q_s
\]
be any nonnegative decomposition of length $R$.  It is minimal, as is the
prescribed decomposition.  Hence Lemma~\ref{lem:exchange-decomposition}(ii)
implies that every irreducible block
\[
(I_k,J_k)
\]
in the exchange
\[
\sum_{r=1}^{R}P_r
=
\sum_{s=1}^{R}Q_s
\]
is balanced:
\[
|I_k|=|J_k|=:p_k.
\]

If some block had $p_k\ge2$, then Theorem~\ref{thm:positive-splitting}
would yield
\[
\beta(I_k)+\tau(I_k)
\le
2p_k-2,
\]
whereas condition~\eqref{eq:U} requires
\[
\beta(I_k)+\tau(I_k)
\ge
2p_k-1.
\]
This is impossible.  Hence every irreducible block is a singleton:
\[
P_r=Q_{s(r)}
\]
for a bijection $r\mapsto s(r)$.  The two decompositions therefore have
the same multiset of rank-one terms and are equivalent.
\end{proof}

\begin{remark}[The two thresholds]
\label{rem:thresholds}
The two parts of Theorem~\ref{thm:main} differ only in the possible size of
an irreducible competing block.  A shorter decomposition necessarily
produces an unbalanced block with $|J_k|\le |I_k|-1$, so the positive
splitting inequality yields the upper bound
\[
\beta(I_k)+\tau(I_k)\le2|I_k|-3.
\]
For two decompositions of the same minimal length, every block is balanced;
a nontrivial block then has
\[
\beta(I_k)+\tau(I_k)\le2|I_k|-2.
\]
This is why the minimality and uniqueness thresholds differ by exactly one.
\end{remark}

\section{Structural Consequences}\label{sec:structure}

The positive-scattering criterion is compatible with several natural
structural operations on a tensor decomposition.  We first consider
\emph{reshaping}, which changes the grouping of the tensor modes and can
increase the dimension budget.  We then consider \emph{appending modes},
which adds nonnegative factors and yields a monotonicity property for the
combined dimension--scattering criterion.

\subsection{Reshaping}\label{subsec:reshaping}

Fix, for each $S\subseteq[R]$, a partition
\[
G_1\sqcup\cdots\sqcup G_t=[d],
\qquad
G_\ell\neq\varnothing,
\]
where the partition may depend on $S$.  Group the factors within each block:
\[
\widehat a_r^{(\ell)}
=
\bigotimes_{j\in G_\ell}a_r^{(j)}
\in
\Rnn^{N_\ell}\setminus\{0\},
\qquad
N_\ell=\prod_{j\in G_\ell}n_j .
\]
Under the canonical isomorphism
\[
\bigotimes_{j=1}^d\R^{n_j}
\cong
\bigotimes_{\ell=1}^t\R^{N_\ell},
\]
the tensor therefore admits the $t$-mode nonnegative decomposition
\[
\T
=
\sum_{r=1}^R
\widehat a_r^{(1)}
\otimes\cdots\otimes
\widehat a_r^{(t)}.
\]
Because every grouped factor is nonzero and nonnegative, the theory of
Sections~\ref{sec:cones}--\ref{sec:tau} applies verbatim to the reshaped
tensor whenever $t\ge2$.  For the partition chosen for a given $S$, write
\[
\widehat\beta_S(\cdot),
\qquad
\widehat\tau_S(\cdot)
\]
for the corresponding dimension budget and scattering term.

The following corollary allows the partition used to certify a given subset
$S$ to be chosen independently of the partitions used for other subsets.

\begin{corollary}[Reshaped positive criterion]
\label{cor:reshaped}
Suppose that for every $S\subseteq[R]$ with $|S|\ge2$, there exists a
partition of the modes into $t\ge2$ nonempty groups such that
\[
\widehat\beta_S(S)+\widehat\tau_S(S)
\ge
2|S|-2.
\]
Then
\[
\rankp(\T)=R.
\]
If the partitions can moreover be chosen so that
\[
\widehat\beta_S(S)+\widehat\tau_S(S)
\ge
2|S|-1,
\]
then the decomposition \eqref{eq:decomposition} is unique among
nonnegative decompositions of length $R$.
\end{corollary}

\begin{proof}
Suppose first that $\rankp(\T)<R$.  By the block decomposition lemma,
there is an irreducible positive exchange
\[
\sum_{r\in S}c_rP_r
=
\sum_{s\in J}Q_s,
\qquad
c_r>0,
\qquad
|S|\ge2,
\qquad
|J|\le |S|-1.
\]
Group the modes according to the partition chosen for this set $S$.  The
equality of the two sums is unchanged by this regrouping, and every term
remains a nonzero nonnegative rank-one tensor in the reshaped tensor
format.  Moreover, irreducibility is a property of the exchange as an
equality of sums and is therefore unchanged by regrouping the modes.

Applying Theorem~\ref{thm:positive-splitting} in the grouped tensor format
gives
\[
\widehat\beta_S(S)+\widehat\tau_S(S)
\le
|S|+|J|-2
\le
2|S|-3,
\]
contradicting the assumed lower bound
\[
\widehat\beta_S(S)+\widehat\tau_S(S)\ge2|S|-2.
\]
Hence $\rankp(\T)=R$.

Now assume the stronger hypothesis and suppose that
\[
\T=\sum_{s=1}^RQ_s
\]
is a nonnegative decomposition of length $R$ that is not equivalent to
\eqref{eq:decomposition}.  By the first part, the prescribed decomposition
is minimal, and the competing decomposition is also minimal.  Hence their
exchange decomposes into irreducible balanced blocks.  Inequivalence implies
that at least one block $(S,J)$ satisfies
\[
|J|=|S|\ge2.
\]
Regroup the modes according to the partition chosen for this $S$.  Applying
Theorem~\ref{thm:positive-splitting} in the grouped format yields
\[
\widehat\beta_S(S)+\widehat\tau_S(S)
\le
|S|+|J|-2
=
2|S|-2,
\]
contradicting the assumed bound
\[
\widehat\beta_S(S)+\widehat\tau_S(S)\ge2|S|-1.
\]
Therefore every nonnegative length-$R$ decomposition is equivalent to
\eqref{eq:decomposition}.
\end{proof}

Thus the certificate need not be evaluated in only the original tensor
format: different subsets $S$ may use different groupings of the modes.
This is useful because grouping can increase the dimensions of the grouped
factor spans even when the individual mode spans are relatively small.

\begin{remark}[The one-block partition]
\label{rem:trivial-partition}
The restriction $t\ge2$ is natural but causes no loss in the criterion.
If all modes are grouped into a single block, then the grouped factors are
the rank-one tensors $P_r$ themselves.  The resulting one-mode budget is at
most $|S|-1$, while the bridge graph is complete because, with only one
mode, every pair is facet-opposite in at most one mode.  Thus the scattering
term is zero and
\[
\widehat\beta_S(S)+\widehat\tau_S(S)
\le
|S|-1
<
2|S|-2
\qquad (|S|\ge2).
\]
Hence the one-block grouping can never by itself certify either main
criterion.
\end{remark}

Reshaping can genuinely strengthen the dimension budget.  A standard
example comes from the Khatri--Rao product, namely the columnwise product
of two matrices \citep[pp.~169--170]{KhatriRao1968}.  Recall that the
Kruskal rank of a matrix is the largest integer $k$ such that every set of
$k$ columns is linearly independent \citep[p.~102]{Kruskal1977}.  If two
factor matrices with $R$ columns have Kruskal ranks whose sum is at least
$R+1$, then their Khatri--Rao product has full column rank
\citep[Lemma~1]{SidiropoulosBro2000}; this may occur even when neither
factor matrix has full column rank.  Thus grouping modes can create linear
independence that is not visible in the individual modes.

We do not establish a general monotonicity relation between the scattering
term before and after reshaping.  In particular, the grouped scattering
term may interact with the changed factor geometry in ways that are not
captured by the dimension budget alone.

\subsection{Appending Modes}\label{subsec:appending}

The effect of adding new nonnegative modes is different.  Here the original
modes are retained, while additional factor vectors are appended to each
term:
\[
\widetilde P_r
=
a_r^{(1)}\otimes\cdots\otimes a_r^{(d)}
\otimes
a_r^{(d+1)}
\otimes\cdots\otimes
a_r^{(d')},
\]
where
\[
a_r^{(j)}\in\Rnn^{n_j}\setminus\{0\},
\qquad
j=d+1,\dots,d'.
\]
Write $\beta_M(S)$ and $\tau_M(S)$ for the budget and scattering term
computed using only the modes in a nonempty set $M\subseteq[d']$.

The dimension budget is additive across disjoint sets of modes.  More
importantly, the scattering term is superadditive.

\begin{proposition}[Superadditivity of scattering]
\label{prop:superadditivity}
Let $I,J\subseteq[d']$ be disjoint and nonempty.  Then, for every
$S\subseteq[R]$ with $|S|\ge2$,
\[
\beta_{I\cup J}(S)
=
\beta_I(S)+\beta_J(S)
\]
and
\[
\tau_{I\cup J}(S)
\ge
\tau_I(S)+\tau_J(S).
\]
Consequently,
\[
\beta_{I\cup J}(S)+\tau_{I\cup J}(S)
\ge
\bigl(\beta_I(S)+\tau_I(S)\bigr)
+
\bigl(\beta_J(S)+\tau_J(S)\bigr).
\]
\end{proposition}

\begin{proof}
The identity for $\beta$ follows immediately from its definition.

For the scattering term, suppose first that
\[
\tau_{I\cup J}(S)<\infty
\]
and let $(W_j)_{j\in I\cup J}$ be a feasible tuple attaining
$\tau_{I\cup J}(S)$.  Its bridge graph
$\Gamma_{I\cup J}$ is connected.

Restrict this tuple to the modes in $I$.  Any edge of
$\Gamma_{I\cup J}$ is facet-opposite in at most one mode among
$I\cup J$, and hence also in at most one mode among $I$.  Therefore the
bridge graph $\Gamma_I$ of the restricted tuple contains
$\Gamma_{I\cup J}$ and is connected.  Thus the restricted tuple is feasible
for $I$, and
\[
\tau_I(S)
\le
\sum_{j\in I}
\bigl(\dim W_j-d_j(S)\bigr).
\]
The same argument for $J$ gives
\[
\tau_J(S)
\le
\sum_{j\in J}
\bigl(\dim W_j-d_j(S)\bigr).
\]
Adding,
\[
\tau_I(S)+\tau_J(S)
\le
\sum_{j\in I\cup J}
\bigl(\dim W_j-d_j(S)\bigr)
=
\tau_{I\cup J}(S).
\]
If $\tau_{I\cup J}(S)=+\infty$, the inequality is immediate.
\end{proof}

This yields the following monotonicity result.

\begin{corollary}[Appending nonnegative modes]
\label{cor:more-modes}
Suppose the decomposition \eqref{eq:decomposition} satisfies condition
\eqref{eq:M}, respectively condition \eqref{eq:U}.  Extend each term by
nonzero nonnegative factors in new modes $d+1,\dots,d'$ as above.  Then the
extended decomposition
\[
\widetilde\T
=
\sum_{r=1}^R\widetilde P_r
\]
satisfies the same condition.  Consequently,
\[
\rankp(\widetilde\T)=R
\]
under \eqref{eq:M}, while under \eqref{eq:U} the extended decomposition is
unique among nonnegative decompositions of length $R$.
\end{corollary}

\begin{proof}
Fix $S\subseteq[R]$ with $|S|\ge2$ and put
\[
I=[d],
\qquad
J=\{d+1,\dots,d'\}.
\]
By Proposition~\ref{prop:superadditivity},
\[
\beta_{[d']}(S)+\tau_{[d']}(S)
\ge
\bigl(\beta_{[d]}(S)+\tau_{[d]}(S)\bigr)
+
\bigl(\beta_J(S)+\tau_J(S)\bigr).
\]
Since
\[
\beta_J(S)\ge0,
\qquad
\tau_J(S)\ge0,
\]
we obtain
\[
\beta_{[d']}(S)+\tau_{[d']}(S)
\ge
\beta_{[d]}(S)+\tau_{[d]}(S).
\]
Thus whichever of the two thresholds is satisfied by the original
decomposition is also satisfied by the extended decomposition.  Applying
Theorem~\ref{thm:main} proves the result.
\end{proof}

The superadditivity result gives a precise sense in which additional
nonnegative modes can only increase the amount of structural information
available to the criterion.  Unlike reshaping, which changes the mode
structure and therefore requires a fresh geometric analysis, appending a
mode preserves the existing certificate and adds a nonnegative contribution
to the combined dimension--scattering budget.

\section{Examples}\label{sec:examples}

This section illustrates the two principal consequences of the
positive-scattering criterion.  We first give explicit deterministic
families showing that the new criterion can certify nonnegative uniqueness
where every dimension-based Lovitz--Petrov criterion fails, even after
reshaping.  We then examine the size of this gain on random sparse
decompositions and search empirically for nonnegative alternatives when the
criterion fails.

\subsection{Explicit Strictness Examples}\label{subsec:strictness}

We begin with two families for which the scattering term is infinite.
Thus the identifiability certificate is driven entirely by support geometry:
the dimension budget may fall strictly below the Lovitz--Petrov threshold,
but nonnegative decompositions cannot exchange mass across the disconnected
support pattern.

\begin{example}[$W$ tensor]\label{ex:W}
Let $n_1=n_2=n_3=2$ and
\[
\T
=
e_1\otimes e_1\otimes e_2
+
e_1\otimes e_2\otimes e_1
+
e_2\otimes e_1\otimes e_1.
\]
This is, up to normalization, the three-qubit $W$ state
\citep[Eq.~(2)]{DurVidalCirac2000}.
\end{example}

For every pair $S$ of terms, two factor vectors are parallel in one mode
and span a two-dimensional space in the other two modes.  Hence
\[
\beta(S)=0+1+1=2,
\]
so the pairwise Lovitz--Petrov uniqueness threshold
$2|S|-1=3$ fails.  For the full set,
\[
d_j([3])=2,\qquad j=1,2,3,
\]
and therefore
\[
\beta([3])=3<5.
\]

The scattering contribution is instead infinite.  The three terms are
coordinate tensors with labels
\[
(1,1,2),\qquad
(1,2,1),\qquad
(2,1,1).
\]
Any two labels differ in exactly two coordinates.  Hence every pair of
terms has disjoint supports in two modes.  By
Theorem~\ref{thm:tau-activation}(iii), the fully activated graph
$G_{[d]}(S)$ has no edges for every $S$ with $|S|\ge2$, so
\[
\tau(S)=+\infty.
\]
Consequently condition~\eqref{eq:U} holds for every nontrivial subset
$S$, and Theorem~\ref{thm:main} gives
\[
\rankp(\T)=3
\]
and uniqueness among nonnegative decompositions of length three.

The failure of all reshaped Lovitz--Petrov conditions can also be checked
explicitly.  For a pair, every partition of the three modes produces a
grouped dimension budget at most $2<3$.  For the full set, the three
two-block partitions produce grouped budgets at most $3<5$, while grouping
all three modes gives budget $2<5$.  Thus no reshaping recovers the
Lovitz--Petrov uniqueness threshold.

The distinction is genuinely caused by nonnegativity.  Over $\R$, the
decomposition is not unique.  Indeed,
\[
(e_1+te_2)^{\otimes3}-(e_1-te_2)^{\otimes3}
=
2t\,\T+2t^3e_2^{\otimes3},
\]
and hence, for every $t>0$,
\[
\T
=
\frac1{2t}(e_1+te_2)^{\otimes3}
-
\frac1{2t}(e_1-te_2)^{\otimes3}
-
t^2e_2^{\otimes3}.
\]
Thus the same tensor admits a continuum of real rank-three
decompositions, while the displayed nonnegative decomposition is unique.
The gain from the positive-scattering criterion is therefore not a
reformulation of unrestricted CP uniqueness.

For completeness, the real rank of $\T$ is also three.  Suppose otherwise
that
\[
\T=\sum_{i=1}^{2}u_i\otimes v_i\otimes w_i.
\]
The two mode-$1$ slices
\[
T_1=
\begin{pmatrix}
0&1\\
1&0
\end{pmatrix},
\qquad
T_2=
\begin{pmatrix}
1&0\\
0&0
\end{pmatrix}
\]
would lie in the two-dimensional span of
$v_1w_1^\top,v_2w_2^\top$.  Since $T_1,T_2$ are linearly independent, that
span would equal
\[
\{\alpha T_1+\beta T_2:\alpha,\beta\in\R\}.
\]
But
\[
\det(\alpha T_1+\beta T_2)=-\alpha^2,
\]
so the rank-one matrices in this pencil form only the one-dimensional
subspace spanned by $T_2$, a contradiction.

The next example shows that the same strictness phenomenon persists for
arbitrary decomposition length.

\begin{example}[Arbitrary nonnegative rank]\label{ex:all-rank}
Fix $R\ge3$, let
\[
n_1=R,\qquad n_2=R-1,\qquad n_3=2,
\]
and define
\[
\T_R
=
\sum_{r=1}^{R-1}e_r\otimes e_r\otimes e_1
+
e_R\otimes e_1\otimes e_2.
\]
\end{example}

Every term is a coordinate tensor.  Any two labels differ in at least two
coordinates: two diagonal labels $(r,r,1)$ and $(r',r',1)$ differ in
modes $1$ and $2$, while $(r,r,1)$ and $(R,1,2)$ differ in modes $1$ and
$3$ (and also in mode $2$ unless $r=1$).  Hence every pair has disjoint
supports in at least two modes.  It follows from
Theorem~\ref{thm:tau-activation}(iii) that
\[
\tau(S)=+\infty
\qquad\text{for every }S\subseteq[R],\quad |S|\ge2.
\]
Condition~\eqref{eq:U} therefore holds for every $S$, and
\[
\rankp(\T_R)=R
\]
with uniqueness among nonnegative decompositions of length $R$.

For the full set $S=[R]$,
\[
d_1(S)=R,\qquad
d_2(S)=R-1,\qquad
d_3(S)=2,
\]
so
\[
\beta([R])
=
(R-1)+(R-2)+1
=
2R-2
<
2R-1.
\]
Thus the unreshaped Lovitz--Petrov condition fails.  It also fails after
every reshaping.  The grouped dimension budgets for the five partitions of
the three modes are
\[
\begin{array}{c|c|c}
\text{partition} & \text{grouped ranks} & \text{budget}\\
\hline
\{1\},\{2\},\{3\}
&
R,\ R-1,\ 2
&
2R-2
\\
\{1,2\},\{3\}
&
R,\ 2
&
R
\\
\{1,3\},\{2\}
&
R,\ R-1
&
2R-3
\\
\{2,3\},\{1\}
&
R,\ R
&
2R-2
\\
\{1,2,3\}
&
R
&
R-1.
\end{array}
\]
The maximum is $2R-2$, still strictly below the uniqueness threshold
$2R-1$.

The example again separates real and nonnegative identifiability.
Flattening $\T_R$ along mode $1$ gives
\[
\sum_{r=1}^{R-1}
e_r(e_r\otimes e_1)^\top
+
e_R(e_1\otimes e_2)^\top,
\]
whose $R$ rows are distinct coordinate vectors.  Thus the real rank is at
least $R$, while the displayed decomposition has length $R$, so its real
rank is exactly $R$.

At the same time, the first $R-1$ terms can be replaced by an arbitrary
rank factorization of the identity.  If
\[
G=[g_1\ \cdots\ g_{R-1}]
\]
is invertible and $h_s$ denotes the $s$th column of $G^{-\top}$, then
\[
I_{R-1}
=
\sum_{s=1}^{R-1}g_sh_s^\top,
\]
and therefore, writing $\iota:\R^{R-1}\to\R^{R}$ for the embedding onto the
first $R-1$ coordinates,
\[
\T_R
=
\sum_{s=1}^{R-1}
\iota(g_s)\otimes h_s\otimes e_1
+
e_R\otimes e_1\otimes e_2
\]
is another real decomposition of length $R$.  These decompositions are
generically inequivalent.  By the classical characterization of matrices
whose inverse is also nonnegative \citep{BermanPlemmons1994}, such a
factorization is nonnegative only in the monomial case.  Thus the real
alternatives do not contradict the nonnegative uniqueness established
above.

Figure~\ref{fig:examples} summarizes the support obstruction common to the
two constructions.

\begin{figure}[htbp]
\centering
\begin{tikzpicture}[scale=0.98,>=Stealth]
\fill[orange!32,fill opacity=0.9] (0.460,0.330) -- (1.460,0.330) -- (1.460,1.330) -- (0.460,1.330) -- cycle;
\fill[orange!16,fill opacity=0.9] (0.460,1.330) -- (1.460,1.330) -- (1.920,1.660) -- (0.920,1.660) -- cycle;
\fill[orange!48,fill opacity=0.9] (1.460,0.330) -- (1.920,0.660) -- (1.920,1.660) -- (1.460,1.330) -- cycle;
\draw[orange!70!black,thin] (0.460,0.330) -- (1.460,0.330);
\draw[orange!70!black,thin] (1.460,0.330) -- (1.460,1.330);
\draw[orange!70!black,thin] (1.460,1.330) -- (0.460,1.330);
\draw[orange!70!black,thin] (0.460,1.330) -- (0.460,0.330);
\draw[orange!70!black,thin] (0.460,1.330) -- (0.920,1.660);
\draw[orange!70!black,thin] (1.460,1.330) -- (1.920,1.660);
\draw[orange!70!black,thin] (0.920,1.660) -- (1.920,1.660);
\draw[orange!70!black,thin] (1.460,0.330) -- (1.920,0.660);
\draw[orange!70!black,thin] (1.920,0.660) -- (1.920,1.660);
\draw[black!40,thin,dashed] (0.920,0.660) -- (2.920,0.660);
\draw[black!40,thin,dashed] (0.920,0.660) -- (0.000,0.000);
\draw[black!40,thin,dashed] (0.920,0.660) -- (0.920,2.660);
\fill[blue!32,fill opacity=0.9] (0.000,1.000) -- (1.000,1.000) -- (1.000,2.000) -- (0.000,2.000) -- cycle;
\fill[blue!16,fill opacity=0.9] (0.000,2.000) -- (1.000,2.000) -- (1.460,2.330) -- (0.460,2.330) -- cycle;
\fill[blue!48,fill opacity=0.9] (1.000,1.000) -- (1.460,1.330) -- (1.460,2.330) -- (1.000,2.000) -- cycle;
\draw[blue!70!black,thin] (0.000,1.000) -- (1.000,1.000);
\draw[blue!70!black,thin] (1.000,1.000) -- (1.000,2.000);
\draw[blue!70!black,thin] (1.000,2.000) -- (0.000,2.000);
\draw[blue!70!black,thin] (0.000,2.000) -- (0.000,1.000);
\draw[blue!70!black,thin] (0.000,2.000) -- (0.460,2.330);
\draw[blue!70!black,thin] (1.000,2.000) -- (1.460,2.330);
\draw[blue!70!black,thin] (0.460,2.330) -- (1.460,2.330);
\draw[blue!70!black,thin] (1.000,1.000) -- (1.460,1.330);
\draw[blue!70!black,thin] (1.460,1.330) -- (1.460,2.330);
\fill[teal!32,fill opacity=0.9] (1.000,0.000) -- (2.000,0.000) -- (2.000,1.000) -- (1.000,1.000) -- cycle;
\fill[teal!16,fill opacity=0.9] (1.000,1.000) -- (2.000,1.000) -- (2.460,1.330) -- (1.460,1.330) -- cycle;
\fill[teal!48,fill opacity=0.9] (2.000,0.000) -- (2.460,0.330) -- (2.460,1.330) -- (2.000,1.000) -- cycle;
\draw[teal!70!black,thin] (1.000,0.000) -- (2.000,0.000);
\draw[teal!70!black,thin] (2.000,0.000) -- (2.000,1.000);
\draw[teal!70!black,thin] (2.000,1.000) -- (1.000,1.000);
\draw[teal!70!black,thin] (1.000,1.000) -- (1.000,0.000);
\draw[teal!70!black,thin] (1.000,1.000) -- (1.460,1.330);
\draw[teal!70!black,thin] (2.000,1.000) -- (2.460,1.330);
\draw[teal!70!black,thin] (1.460,1.330) -- (2.460,1.330);
\draw[teal!70!black,thin] (2.000,0.000) -- (2.460,0.330);
\draw[teal!70!black,thin] (2.460,0.330) -- (2.460,1.330);
\draw[black!25,very thin] (1.000,0.000) -- (1.000,2.000);
\draw[black!25,very thin] (0.000,1.000) -- (2.000,1.000);
\draw[black!25,very thin] (1.000,2.000) -- (1.920,2.660);
\draw[black!25,very thin] (0.460,2.330) -- (2.460,2.330);
\draw[black!25,very thin] (2.460,0.330) -- (2.460,2.330);
\draw[black!25,very thin] (2.000,1.000) -- (2.920,1.660);
\draw[black!60,thin] (0.000,0.000) -- (2.000,0.000);
\draw[black!60,thin] (2.000,0.000) -- (2.000,2.000);
\draw[black!60,thin] (2.000,2.000) -- (0.000,2.000);
\draw[black!60,thin] (0.000,2.000) -- (0.000,0.000);
\draw[black!60,thin] (0.000,2.000) -- (0.920,2.660);
\draw[black!60,thin] (2.000,2.000) -- (2.920,2.660);
\draw[black!60,thin] (0.920,2.660) -- (2.920,2.660);
\draw[black!60,thin] (2.000,0.000) -- (2.920,0.660);
\draw[black!60,thin] (2.920,0.660) -- (2.920,2.660);
\node[blue!70!black,font=\scriptsize,above left] at (0.300,2.000) {$(1,1,2)$};
\node[teal!60!black,font=\scriptsize,anchor=west] at (2.350,0.020) {$(2,1,1)$};
\draw[orange!85!black,thin] (0.960,0.880) -- (-0.75,0.35);
\node[orange!85!black,font=\scriptsize,left] at (-0.72,0.35) {$(1,2,1)$};
\node[font=\scriptsize,align=center] at (1.35,-0.74) {any two labels differ\\[2.5pt]in two coordinates};
\node[font=\small] at (-1.15,3.05) {(a)};
\end{tikzpicture}
\hfill
\begin{tikzpicture}[scale=0.95,>=Stealth]
\fill[black!4] (0.550,1.550) -- (3.050,1.550) -- (4.010,2.750) -- (1.510,2.750) -- cycle;
\fill[orange!45] (2.550,1.550) -- (3.050,1.550) -- (3.290,1.850) -- (2.790,1.850) -- cycle;
\draw[black!30,very thin] (0.550,1.550) -- (1.510,2.750);
\draw[black!30,very thin] (1.050,1.550) -- (2.010,2.750);
\draw[black!30,very thin] (1.550,1.550) -- (2.510,2.750);
\draw[black!30,very thin] (2.050,1.550) -- (3.010,2.750);
\draw[black!30,very thin] (2.550,1.550) -- (3.510,2.750);
\draw[black!30,very thin] (3.050,1.550) -- (4.010,2.750);
\draw[black!30,very thin] (0.550,1.550) -- (3.050,1.550);
\draw[black!30,very thin] (0.790,1.850) -- (3.290,1.850);
\draw[black!30,very thin] (1.030,2.150) -- (3.530,2.150);
\draw[black!30,very thin] (1.270,2.450) -- (3.770,2.450);
\draw[black!30,very thin] (1.510,2.750) -- (4.010,2.750);
\draw[orange!85!black,thin] (2.550,1.550) -- (3.050,1.550) -- (3.290,1.850) -- (2.790,1.850) -- cycle;
\fill[white,fill opacity=0.94] (0.000,0.000) -- (2.500,0.000) -- (3.460,1.200) -- (0.960,1.200) -- cycle;
\fill[black!4] (0.000,0.000) -- (2.500,0.000) -- (3.460,1.200) -- (0.960,1.200) -- cycle;
\fill[blue!40] (0.000,0.000) -- (0.500,0.000) -- (0.740,0.300) -- (0.240,0.300) -- cycle;
\fill[blue!40] (0.740,0.300) -- (1.240,0.300) -- (1.480,0.600) -- (0.980,0.600) -- cycle;
\fill[blue!40] (1.480,0.600) -- (1.980,0.600) -- (2.220,0.900) -- (1.720,0.900) -- cycle;
\fill[blue!40] (2.220,0.900) -- (2.720,0.900) -- (2.960,1.200) -- (2.460,1.200) -- cycle;
\draw[black!45,very thin] (0.000,0.000) -- (0.960,1.200);
\draw[black!45,very thin] (0.500,0.000) -- (1.460,1.200);
\draw[black!45,very thin] (1.000,0.000) -- (1.960,1.200);
\draw[black!45,very thin] (1.500,0.000) -- (2.460,1.200);
\draw[black!45,very thin] (2.000,0.000) -- (2.960,1.200);
\draw[black!45,very thin] (2.500,0.000) -- (3.460,1.200);
\draw[black!45,very thin] (0.000,0.000) -- (2.500,0.000);
\draw[black!45,very thin] (0.240,0.300) -- (2.740,0.300);
\draw[black!45,very thin] (0.480,0.600) -- (2.980,0.600);
\draw[black!45,very thin] (0.720,0.900) -- (3.220,0.900);
\draw[black!45,very thin] (0.960,1.200) -- (3.460,1.200);
\node[font=\scriptsize,black!65,anchor=east] at (-0.140,0.550) {slice $x_3=1$};
\node[font=\scriptsize,black!65,above] at (2.310,2.800) {slice $x_3=2$};
\node[blue!70!black,font=\scriptsize,anchor=west] at (3.050,0.400) {$(r,r,1)$};
\node[orange!85!black,font=\scriptsize,right] at (3.195,1.700) {$(5,1,2)$};
\node[font=\scriptsize,align=center] at (2.25,-0.62) {diagonal vs.\ diagonal: modes $1,2$ differ\\[2.5pt]diagonal vs.\ $(5,1,2)$: modes $1,3$ differ};
\node[font=\small] at (-0.30,3.05) {(b)};
\end{tikzpicture}
\hfill
\begin{tikzpicture}[scale=0.95,>=Stealth]
\node[circle,draw,inner sep=1.3pt,font=\scriptsize] (v1) at (0,0.30) {1};
\node[circle,draw,inner sep=1.3pt,font=\scriptsize] (v2) at (1.7,0.30) {2};
\node[circle,draw,inner sep=1.3pt,font=\scriptsize] (v3) at (0.85,1.70) {3};
\draw[dotted,black!35] (v1) -- (v3);
\draw[thick] (v1) -- (v2);
\draw[thick] (v2) -- (v3);
\node[red!65!black,font=\footnotesize] at (0.85,0.02) {$\infty$};
\node[red!65!black,font=\footnotesize] at (1.58,1.10) {$\infty$};
\node[font=\scriptsize] at (0.85,2.22) {any tree $T$, any $\varepsilon$};
\node[font=\scriptsize,align=center] at (0.85,-1.00)
{each edge keeps an unexempted\\[2.5pt]mode with disjoint supports\\[3.5pt]$\Rightarrow\ \tau(S)=+\infty$};
\node[font=\small] at (-0.62,2.22) {(c)};
\end{tikzpicture}
\caption{The combinatorics behind the infinite scattering of both examples.
(a) The three terms of the $W$ tensor (Example~\ref{ex:W}) as cells of the
$2\times2\times2$ array: any two of the labels $(1,1,2)$, $(1,2,1)$,
$(2,1,1)$ differ in two coordinates, so the corresponding factor pairs
have disjoint supports in two modes.
(b) The terms of $\T_R$ (Example~\ref{ex:all-rank}, drawn for $R=5$) in the two
slices $x_3=1,2$: the diagonal labels $(r,r,1)$ and the isolated label
$(R,1,2)$ again differ pairwise in at least two coordinates.
(c) Two equivalent readings of the obstruction.  In the tree formula
\eqref{eq:tau-tree}, whichever spanning tree and exemption map
$\varepsilon$ one chooses, each tree edge keeps at least one unexempted
mode in which the two supports are disjoint, so
$\kappa_j(F_j(T,\varepsilon))=+\infty$ there by Lemma~\ref{lem:kappa-finite}.
Equivalently, no pair is an edge of the activation graph $G_{[d]}(S)$, so
$G_{[d]}(S)$ is disconnected and Theorem~\ref{thm:tau-activation}(iii) gives
$\tau(S)=+\infty$ for every $|S|\ge2$, whereas every reshaped budget
falls short of the thresholds.}
\label{fig:examples}
\end{figure}
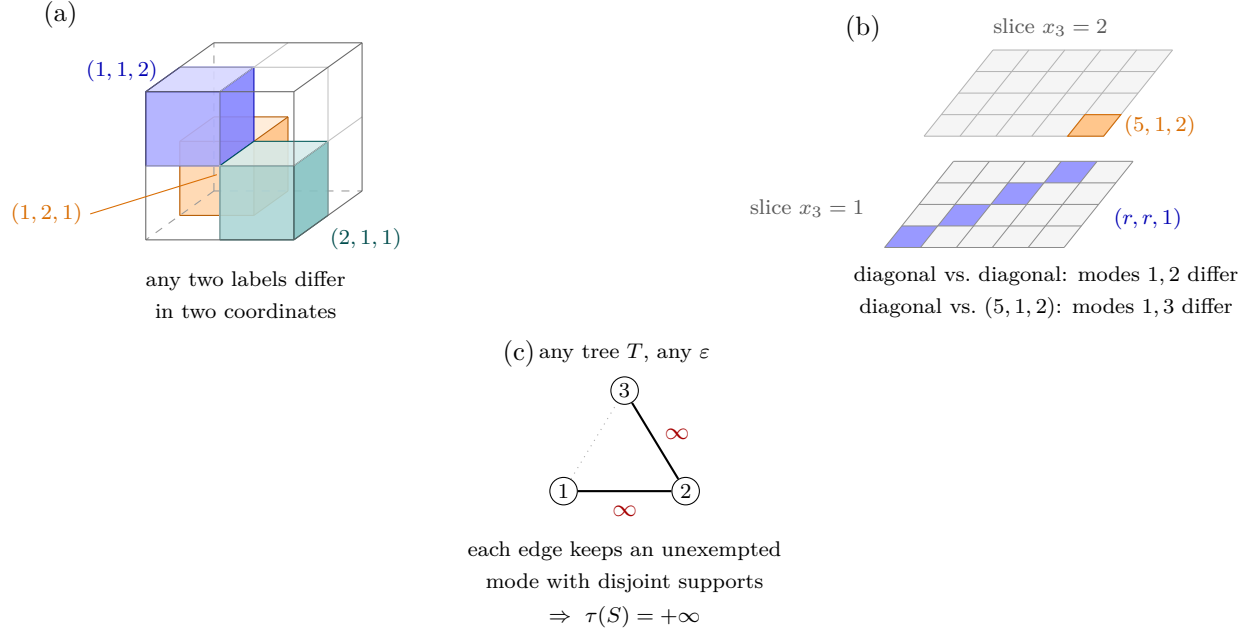

\subsection{Numerical Comparison}\label{subsec:numerics}

We next examine the gain from the scattering term on random sparse
nonnegative decompositions.  Fix $d=3$ and $R=5$, let all mode dimensions
equal $n$, and draw the support of each factor coordinatewise from
$\operatorname{Bernoulli}(p)$, conditioning on nonempty factors.  On each
support, assign independent integer values uniformly from
$\{1,\ldots,999\}$.

For each realization we compute the mode dimensions, the pairwise
scattering costs, and the resulting $\tau(S)$ exactly.  Conditions
\eqref{eq:M} and \eqref{eq:U} are then checked for every nontrivial subset
$S\subseteq[R]$.  Thus the reported certification outcomes do not depend
on numerical tolerances.

Figure~\ref{fig:phase} compares condition~\eqref{eq:U} with Kruskal's
condition \citep[Theorem~4a]{Kruskal1977} and the Lovitz--Petrov condition
\[
\beta(S)\ge2|S|-1
\qquad
\text{for every }S\subseteq[R],\quad |S|\ge2.
\]
Kruskal's condition implies the Lovitz--Petrov condition
\citep[p.~3]{LovitzPetrov2023}, and the latter implies~\eqref{eq:U}
because $\tau(S)\ge0$.  Thus the positive-scattering criterion can only
expand the certified region.

The difference is largest in sparse regimes.  There, support disjointness
creates large or infinite scattering costs even when the factor-span
dimensions remain too small to meet the dimension-only threshold.  At full
support, ordinary support disjointness never occurs, so the infinite-cost
obstruction disappears.

\begin{figure*}[htbp]
\centering
\includegraphics[width=\textwidth]{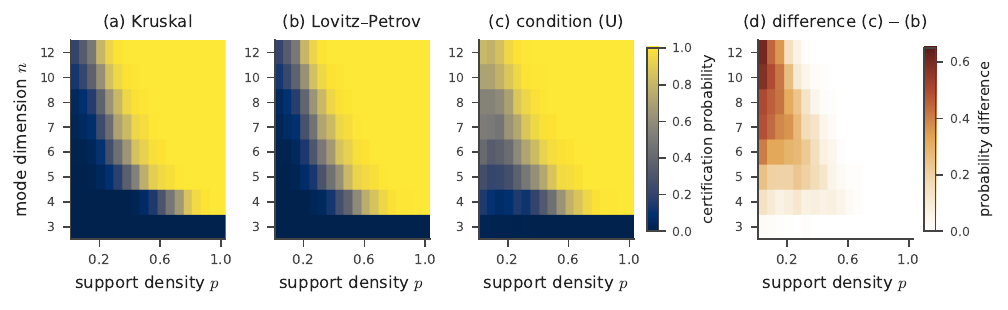}
\caption{Certification probabilities for random nonnegative decompositions
with $d=3$ and $R=5$.  Factor supports are drawn coordinatewise from
$\operatorname{Bernoulli}(p)$, conditioned on being nonempty, and nonzero
entries are independent uniform integers in $\{1,\ldots,999\}$.  Each cell
aggregates $500$ realizations.  (a)~Kruskal's condition
$\sum_j(k_j-1)\ge2R-1$.  (b)~The Lovitz--Petrov condition
$\beta(S)\ge2|S|-1$ for every $S$.  (c)~The positive-scattering condition
\eqref{eq:U}.  (d)~The increase in certification probability from (b) to
(c).  The largest observed increase is $0.60$, at $(n,p)=(12,0.04)$.
At full support, the infinite support-separation obstruction is absent.}
\label{fig:phase}
\end{figure*}

\subsection{Searching for Alternatives}\label{subsec:search}

The previous experiment asks when the criterion certifies uniqueness.  We
now ask the converse question: when \eqref{eq:U} fails, how often can we
find an explicit nonnegative alternative?

We restrict attention to the slice $n=6$ and search every realization that
violates \eqref{eq:U}.  The search proceeds in two stages.  First, we test
explicit pair constructions.  If two prescribed terms can be combined into
a shorter nonnegative rank-one representation, minimality fails.  If they
are nonparallel in exactly two modes and have nested supports in one of
those modes, the identity
\[
x_1\otimes y_1+x_2\otimes y_2
=
x_1\otimes(y_1+t y_2)
+
(x_2-tx_1)\otimes y_2
\]
produces an inequivalent nonnegative decomposition whenever
\[
0<t<
\min_{i\in\supp(x_1)}
\frac{x_2(i)}{x_1(i)}.
\]
All such constructions are checked exactly over $\mathbb Q$.

When no pair construction applies, we use multi-start nonnegative
alternating least squares \citep[p.~1148]{KimParkElden2007} only as a
heuristic to locate candidate alternatives.  A candidate is counted only
after an exact rational decomposition has been constructed and verified.

Among the $5621$ realizations violating \eqref{eq:U}, an exact pair
construction was found in $5497$, and an additional $10$ realizations
admitted an exact construction on a larger subset.  The remaining $114$
cases are inconclusive.  Thus an exact nonnegative alternative was found in
\[
\frac{5497+10}{5621}\approx0.98
\]
of the violating instances.  This provides empirical evidence that the
criterion may be close to necessary for this sparse ensemble, although the
experiment does not establish necessity.

\begin{figure*}[htbp]
\centering
\includegraphics[width=\textwidth]{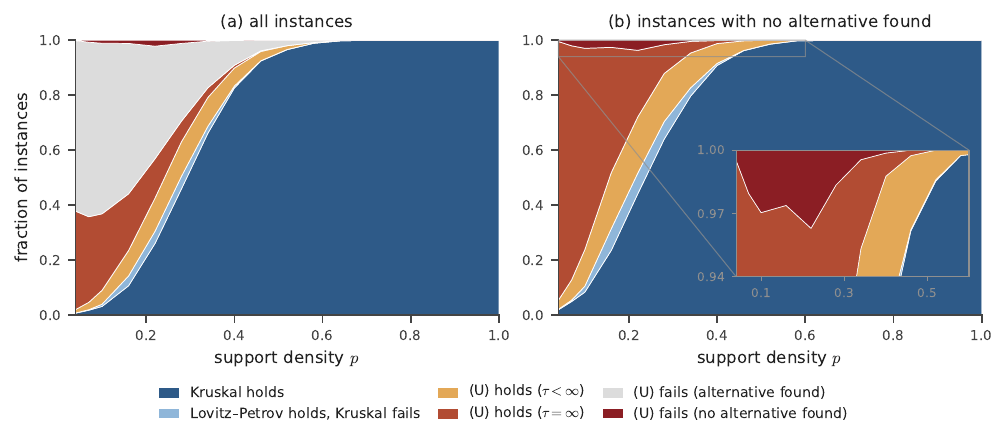}
\caption{Search for nonnegative alternatives on the $n=6$ slice of the
ensemble in Figure~\ref{fig:phase}, using $1600$ realizations for each
value of $p$.  Every realization violating \eqref{eq:U} is searched.
Exact pair constructions and larger-subset constructions are verified
over $\mathbb Q$; candidate decompositions found by alternating least
squares are not counted unless an exact decomposition is subsequently
verified.  (a)~The search outcomes.  (b)~The same data after removing the
exact-alternative layer.  The unresolved cases are therefore genuinely
inconclusive rather than evidence of uniqueness.}
\label{fig:gap}
\end{figure*}

\section{Matrix Specialization}\label{sec:matrix-case}

The order-two case provides a useful boundary case for the general theory.
Identifying $\T$ with a matrix $X$, write
\begin{equation}
X=\T=\sum_{r=1}^{R}a_rb_r^{\top}=AB^{\top},
\qquad
A=[a_1\ \cdots\ a_R]\in\Rnn^{n_1\times R},
\qquad
B=[b_1\ \cdots\ b_R]\in\Rnn^{n_2\times R},
\label{eq:matrix-decomposition}
\end{equation}
where every column of $A$ and $B$ is nonzero. We call a set of columns of a matrix a \emph{circuit} if it is
linearly dependent and every proper subset of it is linearly independent \citep[p.~510]{Whitney1935}. For a nonempty
$S\subseteq[R]$, let $A_S$ and $B_S$ denote the corresponding column
submatrices.  Then
\[
d_1(S)=\rank A_S,
\qquad
d_2(S)=\rank B_S,
\]
and therefore
\begin{equation}
\beta(S)
=
\rank A_S+\rank B_S-2.
\label{eq:matrix-beta}
\end{equation}

The matrix case is especially revealing because neither Kruskal's nor
Lovitz--Petrov's dimension budget can reach the thresholds in
Theorem~\ref{thm:main}.  Hence the matrix content of the present criterion
comes entirely from the positivity-induced scattering term.

\begin{remark}[Dimension budgets in two modes]
\label{rem:matrix-budget}
The usual Kruskal and Lovitz--Petrov theorems are formulated for tensors
with at least three modes.  Here we only examine the corresponding
dimension-budget inequalities after formally setting $d=2$.  For every
$S$ with $|S|\ge2$,
\[
\beta(S)
\le
2|S|-2
<
2|S|-1,
\]
so the Lovitz--Petrov uniqueness threshold can never hold in two modes.
Likewise, if $k_1$ and $k_2$ are the Kruskal ranks of the two factor
matrices, then
\[
(k_1-1)+(k_2-1)
\le
2R-2
<
2R-1,
\]
so the corresponding Kruskal threshold is also unattainable.

Thus, in the matrix case, the budget alone cannot certify uniqueness.
The difference between the minimality and uniqueness criteria is entirely
accounted for by the positive-scattering term.
\end{remark}

Condition~\eqref{eq:M} requires
\[
\tau(S)
\ge
\bigl(|S|-\rank A_S\bigr)
+
\bigl(|S|-\rank B_S\bigr),
\]
while condition~\eqref{eq:U} requires one additional unit.  The next
results show that these inequalities have particularly simple
interpretations.  Minimality reduces exactly to ordinary matrix rank,
whereas uniqueness reduces exactly to two-sided separability.

\subsection{Overlap Graphs and Matrix Scattering}
\label{subsec:matrix-graphs}

For a pair $\{r,t\}\subseteq S$, the two mode costs
$\kappa_1(\{r,t\})$ and $\kappa_2(\{r,t\})$ are computed relative to the
set $S$, using the spaces and support hulls defined in Section~\ref{sec:tau}.
By Proposition~\ref{prop:kappa-trichotomy}, each belongs to
$\{0,1,+\infty\}$.  The cases $0$ and $+\infty$ can be expressed directly
through two natural graphs.

\begin{definition}[Overlap graphs]
\label{def:matrix-overlap-graphs}
Let $S\subseteq[R]$ with $|S|\ge2$.  The \emph{facet-overlap graph}
$\mathsf F_A(S)$ has vertex set $S$ and an edge $\{r,t\}$ whenever
\[
\mathcal A_{U_1(S)}(a_r)
\cap
\mathcal A_{U_1(S)}(a_t)
\neq\varnothing .
\]
The \emph{support-overlap graph} $\mathsf O_A(S)$ has vertex set $S$ and
an edge $\{r,t\}$ whenever
\[
\supp(a_r)\cap\supp(a_t)\neq\varnothing .
\]
Define $\mathsf F_B(S)$ and $\mathsf O_B(S)$ analogously for the columns
of $B$.

All graph unions below are taken on the common vertex set $S$.
\end{definition}

By Proposition~\ref{prop:kappa-trichotomy},
\[
\kappa_1(\{r,t\})=0
\quad\Longleftrightarrow\quad
\{r,t\}\in E(\mathsf F_A(S)),
\]
while Lemma~\ref{lem:kappa-finite} gives
\[
\kappa_1(\{r,t\})<+\infty
\quad\Longleftrightarrow\quad
\{r,t\}\in E(\mathsf O_A(S)).
\]
The corresponding statements hold for $B$.  In particular,
\[
\mathsf F_A(S)\subseteq\mathsf O_A(S),
\qquad
\mathsf F_B(S)\subseteq\mathsf O_B(S).
\]

\begin{proposition}[Matrix activation formula]
\label{prop:matrix-activation}
For every $S\subseteq[R]$ with $|S|\ge2$, the activation graphs
$G_M(S)$ of Section~\ref{subsec:activation} are
\begin{align}
G_{\varnothing}(S)
&=
\mathsf F_A(S)\cup\mathsf F_B(S),
\label{eq:matrix-G0}\\
G_{\{1\}}(S)
&=
\mathsf O_A(S)\cup\mathsf F_B(S),
\label{eq:matrix-G1}\\
G_{\{2\}}(S)
&=
\mathsf F_A(S)\cup\mathsf O_B(S),
\label{eq:matrix-G2}\\
G_{\{1,2\}}(S)
&=
\mathsf O_A(S)\cup\mathsf O_B(S).
\label{eq:matrix-G12}
\end{align}
Consequently,
\begin{equation}
\tau(S)=
\begin{cases}
	0,
	&
	\mathsf F_A(S)\cup\mathsf F_B(S)
	\text{ is connected},\\[2pt]
	1,
	&
	\text{otherwise, if }
	\mathsf O_A(S)\cup\mathsf F_B(S)
	\text{ or }
	\mathsf F_A(S)\cup\mathsf O_B(S)
	\text{ is connected},\\[2pt]
	2,
	&
	\text{otherwise, if }
	\mathsf O_A(S)\cup\mathsf O_B(S)
	\text{ is connected},\\[2pt]
	+\infty,
	&
	\text{otherwise}.
\end{cases}
\label{eq:matrix-tau-four-values}
\end{equation}
In particular,
\[
\tau(S)\in\{0,1,2,+\infty\}.
\]
For a pair $S=\{r,t\}$,
\begin{equation}
\tau(\{r,t\})
=
\min\bigl\{
\kappa_1(\{r,t\}),
\kappa_2(\{r,t\})
\bigr\}.
\label{eq:matrix-pair-tau}
\end{equation}
\end{proposition}

\begin{proof}
For $M\subseteq\{1,2\}$, recall from
Theorem~\ref{thm:tau-activation}(i) that a pair
$e=\{r,t\}$ is an edge of $G_M(S)$ precisely when
\[
\#\Bigl\{
j:\kappa_j(e)=+\infty
\text{ or }
\bigl(\kappa_j(e)=1\text{ and }j\notin M\bigr)
\Bigr\}
\le1.
\]

If $M=\varnothing$, the counted modes are exactly those with
$\kappa_j(e)\ge1$.  Hence $e$ is an edge if and only if at least one of the
two costs is zero, which gives \eqref{eq:matrix-G0}.  If $M=\{1\}$, mode
$1$ contributes to the count only when $\kappa_1(e)=+\infty$, while mode
$2$ contributes whenever $\kappa_2(e)\ge1$.  Thus $e$ is an edge precisely
when $\kappa_1(e)<+\infty$ or $\kappa_2(e)=0$, giving
\eqref{eq:matrix-G1}.  The case $M=\{2\}$ is symmetric, and for
$M=\{1,2\}$ only the infinite costs remain, giving
\eqref{eq:matrix-G12}.

The four graphs are monotone in $M$ because
\[
\mathsf F_A\subseteq\mathsf O_A,
\qquad
\mathsf F_B\subseteq\mathsf O_B.
\]
The activation formula now follows directly from
Theorem~\ref{thm:tau-activation}(i).  For $S=\{r,t\}$, the tree formula
\eqref{eq:tau-pair} leaves exactly one mode nonexempted, giving
\eqref{eq:matrix-pair-tau}.
\end{proof}

\subsection{Row-Separability}\label{subsec:matrix-separability}

We next identify the matrix condition associated with persistent failure
of facet-overlap connectivity.

\begin{definition}[Row-separability]
\label{def:row-separable}
Let $A\in\Rnn^{n_1\times R}$ and let $S\subseteq[R]$ be nonempty.  A row
index $i$ is \emph{pure on $r$ relative to $S$} if
\[
A_{ir}>0,
\qquad
A_{it}=0
\quad\text{for every }t\in S\setminus\{r\}.
\]
The matrix $A$ is \emph{row-separable on $S$} if every $r\in S$ has a row
that is pure on $r$ relative to $S$, and \emph{row-separable} if it is
row-separable on $[R]$.  The factorization \eqref{eq:matrix-decomposition}
is \emph{two-sided separable} if both $A$ and $B$ are row-separable.
\end{definition}

In the conventional notation $X=WH$ with $W=A$ and $H=B^\top$,
row-separability of $B$ means that for every $r$ some column of $H$ is a positive multiple of the coordinate vector $e_r$. This is the separability
condition of Donoho and Stodden \citep{DonohoStodden2004}; row-separability
of $A$ is the analogous condition for $X^\top=BA^\top$.

\begin{proposition}[Facet characterization of separability]
\label{prop:matrix-separability}
Let $A$ have full column rank.  The following are equivalent:
\begin{enumerate}[(i)]
\item $A$ is row-separable;
\item $\mathsf F_A(S)$ is disconnected for every
$S\subseteq[R]$ with $|S|\ge2$;
\item $\mathsf F_A(S)$ is edgeless for every
$S\subseteq[R]$ with $|S|\ge2$.
\end{enumerate}
The same equivalences hold with $A$ and $\mathsf F_A$ replaced by $B$ and
$\mathsf F_B$.
\end{proposition}

\begin{proof}
Because $A$ has full column rank, every $A_S$ has full column rank.  Fix
$S\subseteq[R]$ with $|S|\ge2$ and put
\[
U=U_1(S).
\]
For each active row $i\in N_1(S)$ define the normalized row
\[
q_i
=
\frac{(A_{ir})_{r\in S}}
{\sum_{r\in S}A_{ir}}
\in\R^S,
\]
and let
\[
P=\conv\{q_i:i\in N_1(S)\}.
\]
By Proposition~\ref{prop:normalized-dual-polytope}, a row $i$ defines a
facet of $C(U)$ precisely when $q_i$ is a vertex of $P$.  Moreover,
two columns $a_r,a_t$ have intersecting facet signatures if and only if
some vertex of $P$ has both $r$th and $t$th coordinates positive.  Thus
\begin{equation}
\{r,t\}\in E(\mathsf F_A(S))
\quad\Longleftrightarrow\quad
\text{$P$ has a vertex with positive $r$th and $t$th coordinates.}
\label{eq:matrix-facet-polytope}
\end{equation}

We use this characterization in both directions.

\emph{Row-separable implies edgeless.}
Suppose $A$ is row-separable.  For every $r\in S$, choose a row that is
pure on $r$ relative to $[R]$.  Since $A$ has full column rank, such a row
gives the normalized vector $q_i=e_r$.  Hence the convex hull $P$
contains all coordinate vectors $e_r$, while every $q_i$ is a probability
vector on $S$ and therefore lies in their simplex.  Thus
\[
P=\conv\{e_r:r\in S\},
\]
whose vertices have exactly one positive coordinate.  By
\eqref{eq:matrix-facet-polytope}, $\mathsf F_A(S)$ is edgeless.

\emph{Edgeless implies disconnected.}
This is immediate because $|S|\ge2$.

\emph{Disconnected for every $S$ implies row-separable.}
Suppose instead that $A$ is not row-separable.  Choose an inclusion-minimal
nonempty subset $S\subseteq[R]$ on which $A$ is not row-separable.  Then
$|S|\ge2$.  Choose $r\in S$ for which no row is pure on $r$ relative to $S$.

For every $t\in S\setminus\{r\}$, minimality of $S$ implies that $A$ is
row-separable on $S\setminus\{t\}$.  Hence there exists a row $i_t$ such
that
\[
A_{i_tr}>0,
\qquad
A_{i_tu}=0
\quad
(u\in S\setminus\{r,t\}).
\]
Because the row is not pure on $r$ relative to $S$, necessarily
$A_{i_tt}>0$.  Therefore
\[
q_{i_t}\in\relint\conv\{e_r,e_t\}.
\]

We next claim that $e_r\notin P$.  Otherwise $e_r$ would be a convex
combination of the points $q_i$.  Since all coordinates are nonnegative
and $e_r$ vanishes outside coordinate $r$, every $q_i$ appearing with
positive weight would have to equal $e_r$.  The corresponding row would be
pure on $r$ relative to $S$, a contradiction.

Write $q_{i_t}$ as a convex combination of vertices of $P$.  Since the
coordinates outside $\{r,t\}$ vanish, every vertex appearing with positive
weight lies on $\conv\{e_r,e_t\}$.  Because $q_{i_t}$ has both its $r$th
and $t$th coordinates positive, it cannot be represented using only the
vertex $e_t$.  Moreover, $e_r\notin P$, so no vertex in the representation
can equal $e_r$.  Hence at least one vertex has both the $r$th and $t$th
coordinates positive.  By
\eqref{eq:matrix-facet-polytope},
\[
\{r,t\}\in E(\mathsf F_A(S)).
\]
This holds for every $t\in S\setminus\{r\}$, so $r$ is adjacent to every
other vertex and $\mathsf F_A(S)$ is connected, contradicting the
assumption.  Hence $A$ is row-separable.
\end{proof}

\subsection{Complete Matrix Characterization}\label{subsec:matrix-reduction}

We can now identify both main criteria completely.

\begin{theorem}[Matrices: complete characterization]
\label{thm:matrix-reduction}
For the nonnegative matrix decomposition \eqref{eq:matrix-decomposition}:
\begin{enumerate}[(i)]
\item condition~\eqref{eq:M} holds if and only if
\[
\rank A=\rank B=R,
\]
equivalently,
\[
\rank X=R;
\]
\item condition~\eqref{eq:U} holds if and only if the factorization is
two-sided separable.
\end{enumerate}
\end{theorem}

\begin{proof}
\emph{Part (i).}
If $\rank A=\rank B=R$, then every $A_S$ and $B_S$ has full column rank.
Hence
\[
\beta(S)
=
2|S|-2
\]
for every $S$ with $|S|\ge2$, and condition~\eqref{eq:M} follows from
$\tau(S)\ge0$.

Conversely, suppose $\rank A<R$.  Choose a circuit
$S\subseteq[R]$ among the columns of $A$.  Then
\[
\rank A_S=|S|-1,
\qquad
|S|\ge2.
\]
Let $U=U_1(S)$.  By circuit minimality, there is a linear dependence
\[
\sum_{r\in S}c_ra_r=0
\]
with every $c_r\neq0$.  Applying the injective map
$L_U$ gives
\[
\sum_{r\in S}c_rL_Ua_r=0.
\]
If $\mathsf F_A(S)$ were disconnected, choose a nonempty proper connected
component $C\subsetneq S$.  For every $r\in C$ and
$t\in S\setminus C$, the supports of $L_Ua_r$ and $L_Ua_t$ are disjoint.
Evaluating the linear relation coordinatewise therefore shows
\[
\sum_{r\in C}c_rL_Ua_r=0.
\]
Injectivity of $L_U$ gives
\[
\sum_{r\in C}c_ra_r=0,
\]
contradicting the fact that $S$ is a circuit.  Thus
$\mathsf F_A(S)$ is connected.  Therefore
\[
\mathsf F_A(S)\cup\mathsf F_B(S)
\]
is connected, so Proposition~\ref{prop:matrix-activation} and
Theorem~\ref{thm:tau-activation}(ii) give
\[
\tau(S)=0.
\]
Since
\[
\rank A_S=|S|-1,
\qquad
\rank B_S\le|S|,
\]
we have
\[
\beta(S)
\le
2|S|-3,
\]
and hence
\[
\beta(S)+\tau(S)
\le
2|S|-3
<
2|S|-2.
\]
Thus condition~\eqref{eq:M} fails.  The same argument with $A$ and $B$
interchanged shows that \eqref{eq:M} implies
\[
\rank A=\rank B=R.
\]

Finally,
\[
\rank X\le\min\{\rank A,\rank B\}\le R.
\]
If $\rank A=\rank B=R$, choose left inverses
$L_AA=I_R$ and $L_BB=I_R$.  Then
\[
L_AXL_B^\top
=
L_AAB^\top L_B^\top
=
I_R,
\]
so $\rank X\ge R$, and therefore $\rank X=R$.  Conversely,
$\rank X=R$ forces both $\rank A$ and $\rank B$ to equal $R$.

\emph{Part (ii).}
By Part (i), condition~\eqref{eq:U} can be rewritten as
\[
\tau(S)\ge1
\qquad
\text{for every }S\subseteq[R],\ |S|\ge2,
\]
because then
\[
\beta(S)=2|S|-2.
\]
By Proposition~\ref{prop:matrix-activation},
\[
\tau(S)=0
\quad\Longleftrightarrow\quad
\mathsf F_A(S)\cup\mathsf F_B(S)
\text{ is connected}.
\]
Thus \eqref{eq:U} holds if and only if
\[
\mathsf F_A(S)\cup\mathsf F_B(S)
\]
is disconnected for every $S$ with $|S|\ge2$.

Since each $\mathsf F_A(S)$ and $\mathsf F_B(S)$ is a subgraph of this
union, both must themselves be disconnected.  Proposition~\ref{prop:matrix-separability}
therefore implies that $A$ and $B$ are row-separable.

Conversely, suppose $A$ and $B$ are row-separable.  Then each has full
column rank, since the pure rows for the $R$ columns are necessarily
distinct and produce a positive diagonal $R\times R$ submatrix.
By Proposition~\ref{prop:matrix-separability},
\[
\mathsf F_A(S)
\quad\text{and}\quad
\mathsf F_B(S)
\]
are edgeless for every $|S|\ge2$.  Hence their union is disconnected, so
$\tau(S)\ge1$ by Proposition~\ref{prop:matrix-activation}.  Since
$\beta(S)=2|S|-2$, we obtain
\[
\beta(S)+\tau(S)\ge2|S|-1,
\]
which is condition~\eqref{eq:U}.
\end{proof}

\subsection{Separability and Explicit Recovery}\label{subsec:matrix-recovery}

The matrix uniqueness criterion therefore has a completely observable
form.

\begin{corollary}[Separability, diagonal pattern, and recovery]
\label{cor:matrix-identifiability}
For the nonnegative matrix decomposition \eqref{eq:matrix-decomposition},
the following are equivalent:
\begin{enumerate}[(i)]
\item condition~\eqref{eq:U} holds;
\item the factorization is two-sided separable;
\item there exist pairwise distinct row indices
$i_1,\dots,i_R$ and pairwise distinct column indices
$j_1,\dots,j_R$ such that
\begin{equation}
	X_{i_rj_r}>0,
	\qquad
	X_{i_rj_t}=0,
	\qquad
	r\neq t.
	\label{eq:matrix-diagonal-pattern}
\end{equation}
\end{enumerate}
Under these conditions,
\[
\rankp(X)=R,
\]
and every nonnegative length-$R$ decomposition is equivalent to
\eqref{eq:matrix-decomposition}.  Moreover, if $i_r$ and $j_r$ are pure
indices for $r$ in $A$ and $B$, respectively, then
\begin{equation}
a_rb_r^\top
=
\frac{(Xe_{j_r})(e_{i_r}^{\top}X)}
{X_{i_rj_r}},
\qquad
r\in[R],
\label{eq:matrix-term-readout}
\end{equation}
and hence
\[
X
=
\sum_{r=1}^{R}
\frac{(Xe_{j_r})(e_{i_r}^{\top}X)}
{X_{i_rj_r}}.
\]
\end{corollary}

\begin{proof}
The equivalence of (i) and (ii) is Theorem~\ref{thm:matrix-reduction}.

\emph{(ii) implies (iii).}
For every $r$, choose a row $i_r$ of $A$ and a row $j_r$ of $B$ that are
pure on $r$.  These indices are pairwise distinct within each factor.
Writing
\[
\alpha_r=A_{i_rr}>0,
\qquad
\gamma_r=B_{j_rr}>0,
\]
we obtain
\[
X_{i_rj_t}
=
\sum_{u=1}^{R}
A_{i_ru}B_{j_tu}
=
\alpha_r\gamma_r\,\delta_{rt},
\]
which gives \eqref{eq:matrix-diagonal-pattern}.

\emph{(iii) implies (ii).}
Assume \eqref{eq:matrix-diagonal-pattern}.  For every $r$,
\[
X_{i_rj_r}
=
\sum_{u=1}^{R}A_{i_ru}B_{j_ru}
>0,
\]
so there exists at least one $s(r)$ such that
\[
A_{i_r,s(r)}>0,
\qquad
B_{j_r,s(r)}>0.
\]
The map $r\mapsto s(r)$ is injective.  Indeed, if $s(r)=s(t)=s$ for
$r\neq t$, then the term indexed by $s$ contributes
\[
A_{i_rs}B_{j_ts}>0
\]
to $X_{i_rj_t}$, contradicting the zero pattern.  Hence, after relabeling,
we may assume
\[
s(r)=r.
\]

For $t\neq r$, the diagonal pattern gives
\[
0=X_{i_rj_t}
\ge
A_{i_rt}B_{j_tt}.
\]
Since $B_{j_tt}>0$, we obtain
\[
A_{i_rt}=0.
\]
Thus row $i_r$ of $A$ is pure on $r$.  Symmetrically,
\[
0=X_{i_tj_r}
\ge
A_{i_tt}B_{j_rt}
\]
and $A_{i_tt}>0$ imply that row $j_r$ of $B$ is pure on $r$.  Hence the
factorization is two-sided separable.

The rank and uniqueness statements follow from Theorem~\ref{thm:main}.

\emph{Recovery.}
Since row $j_r$ of $B$ is pure on $r$,
\[
Xe_{j_r}
=
\sum_{u=1}^{R}a_uB_{j_ru}
=
\gamma_ra_r.
\]
Similarly, since row $i_r$ of $A$ is pure on $r$,
\[
e_{i_r}^{\top}X
=
\sum_{u=1}^{R}A_{i_ru}b_u^\top
=
\alpha_rb_r^\top.
\]
Finally,
\[
X_{i_rj_r}
=
\alpha_r\gamma_r.
\]
Substitution gives
\[
\frac{(Xe_{j_r})(e_{i_r}^{\top}X)}
{X_{i_rj_r}}
=
a_rb_r^\top,
\]
which proves \eqref{eq:matrix-term-readout}.
\end{proof}

\begin{remark}[Fast verification in the matrix case]
\label{rem:matrix-fast-verification}
The matrix specialization eliminates the subset enumeration appearing in
the general criterion.  Condition~\eqref{eq:M} is equivalent to
\[
\rank X=R,
\]
and condition~\eqref{eq:U} is equivalent to two-sided separability, which
can be checked directly from the factor matrices.  Thus no enumeration of
subsets and no linear-programming vertex tests are needed for the matrix
certificates.
\end{remark}

%
\section{Conclusion}\label{sec:conclusion}

This paper develops a deterministic identifiability theory for nonnegative
tensor decompositions based on two complementary sources of information.
The Lovitz--Petrov dimension budget captures the linear-algebraic
constraints imposed by the factor spans, while the positive scattering term
captures additional rigidity created by nonnegativity and support geometry.
The positive splitting inequality combines these two effects and yields
separate thresholds for minimality and uniqueness.  Although the scattering
term is defined through an optimization over intermediate factor spaces, its
mode costs reduce exactly to $\{0,1,+\infty\}$, leading to a finite activation
problem on a graph.  The resulting criterion can strictly improve upon
dimension-based uniqueness conditions, including in sparse examples where
reshaping does not recover the Lovitz--Petrov condition.

The theory also clarifies how identifiability behaves under natural
structural operations.  Appending nonnegative modes can only strengthen the
combined dimension--scattering criterion, while reshaping provides
additional flexibility by changing the grouping of the modes.  In the
matrix boundary case, the two criteria admit exact closed-form
interpretations: minimality reduces to ordinary full-rank factorization,
while uniqueness reduces to two-sided separability and admits an explicit
term-recovery formula.

Several directions remain open.  The full certificate still requires
checking all nontrivial subsets of components, motivating the search for
more economical sufficient conditions or algorithms that exploit additional
structure.  The interaction between scattering and reshaping is also not
fully understood: while grouping can increase the dimension budget, we do
not yet have a general comparison between grouped and ungrouped scattering
terms.  Finally, the matrix specialization identifies two-sided
separability as the exact boundary case of the present criterion, leaving
open the question of which broader nonnegative matrix identifiability
phenomena admit genuine higher-order analogues.  More generally, the
positive-scattering perspective suggests a broader program of incorporating
structural constraints beyond linear independence into deterministic
identifiability theory.

\bibliographystyle{plainnat}

\end{document}